\documentclass{article}
\usepackage{amsmath,amsfonts}

\usepackage[ruled,linesnumbered]{algorithm2e}
\usepackage{algorithmic}
\usepackage{amsmath}
\usepackage{amsfonts}
\usepackage{amssymb}
\usepackage{amsthm}
\usepackage{bbm}
\usepackage{hyperref}
\usepackage{mathtools}
\usepackage{mathrsfs}
\usepackage{booktabs} 
\mathtoolsset{showonlyrefs}
\usepackage[finalold]{trackchanges}
\usepackage{titling}
\newcounter{theorem}

\newcounter{lem_counter}

\newcounter{ass_counter}

\newtheorem{lemma}[lem_counter]{Lemma}

\newtheorem{assumption}[ass_counter]{Assumption}

\usepackage{array}
\usepackage[caption=false,font=normalsize,labelfont=sf,textfont=sf]{subfig}
\usepackage{textcomp}
\usepackage{stfloats}
\usepackage{url}
\usepackage{verbatim}
\usepackage{graphicx}
\usepackage{cite}
\begin{document}

\title{Appendix to STOPS: Short-term Volatility-controlled Policy Search and its Global Convergence}

\author{Liangliang Xu, Daoming Lyu, Yangchen Pan, Aiwen Jiang, Bo Liu
}
\date{}
\maketitle
\let\clearpage\relax
\include{mainref}
\setcounter{page}{1}
\setcounter{section}{0}
\section{Notation Systems}
\label{sec:notations}
\begin{itemize}
    \item $(\mathcal{S}, \mathcal{A}, \mathcal{P}, r,\gamma)$ with state space $\mathcal{S}$, action space $\mathcal{A}$, the transition kernel $\mathcal{P}$, the reward function $r$, the initial state $S_0$ and its distribution $\mu_{0}$, and the discounted factor $\gamma$.
    \item $r_{\max} > 0$ is a constant as the upper bound of the reward. 
    \item State value function $V_{\pi}(s)$ and state-action value function $Q_{\pi}(s,a)$.
    \item The normalized state and state action occupancy measure of policy $\pi$ is denoted by $\nu_\pi(s)$ and $\sigma_\pi(s,a)$
    \item $T$ is the length of a \textit{trajectory}.
    \item The return is defined as $G$. $J(\pi)$ is the expectation of $G$.
    \item Policy $\pi_\theta$ is parameterized by the parameter $\theta$. 
    \item $\tau$ is the temperature parameter in the softmax parameterization of the policy.
    \item $F(\theta)$ is the Fisher information matrix.
    \item $\eta_TD$ is the learning rate of TD update. Similarly, $\eta_NPG$ is the learning rate of NPG update. $\eta_PPO$ is the learning rate of PPO update.
    \item $\beta$ is the penalty factor of KL difference in PPO update.
    \item $f\big((s,a);\theta\big)$ is the two-layer over-parameterized neural network, with $m$ as its width.
    \item $\phi_\theta$ is the feature mapping of the neural network.
    \item $\mathcal{D}$ is the parameter space for $\theta$, with $\Upsilon$ as its radius.
    \item $M >0$ is a constant as the initialization upper bound on $\theta$.
    \item $J^G_\lambda (\pi)$ is the mean-variance objective function.
    \item $J_\lambda (\pi)$ is the reward-volatility objective function, with $\lambda$ as the penalty factor.
    \item $J_\lambda^y(\pi)$ is the transformed reward-volatility objective function, with $y$ as the auxiliary variable.
    \item $\tilde{r}$ is the reward for the augmented MDP. Similarly, $\tilde V_\pi(s)$ and $\tilde Q_\pi(s,a)$ are state value function and state-action value function of the augmented MDP, respectively. $\tilde J(\pi)$ is the risk-neural objective of the augmented MDP.
    \item $\hat{y}_{k}$ is an estimator of $y$ at $k$-th iteration.
    \item $\omega$ is the parameter of critic network.
    \item $\delta_k=\text{argmin}_{\delta\in\mathcal{D}}\Vert\hat{F}(\theta_k)\delta-\tau_k\hat{\nabla}_\theta J(\pi_{\theta_k} )\Vert_2$.
    \item $\xi_k(\delta)=\hat{F}(\theta_k)\delta-\tau_k\hat{\nabla}_\theta \Tilde{J}(\pi_{\theta_k})-\mathbb{E}[\hat{F}(\theta_k)\delta-\tau_k\hat{\nabla}_\theta \Tilde{J}(\pi_{\theta_k} )]$.
    \item $\sigma_\xi$ is a constant associated with the upper bound of the gradient variance.
    \item $\varphi_k,\psi_k,\varphi'_k,\psi'_k$ are the concentability coefficients, upper bounded by a constant $c_0 > 0$.
    \item $\varphi^*_{k} = \mathbb{E}_{(s,a) \sim \sigma_\pi}\bigg[\big(\frac{d\pi^*}{d\pi_0}-\frac{d\pi_{\theta_k}}{d\pi_0}\big)^2\bigg]^{1/2}$.
    \item $\psi^*_{k} = \mathbb{E}_{(s,a) \sim \sigma_\pi}\bigg[\big(\frac{d\sigma_{\pi^*}}{d\sigma_\pi}-\frac{d\nu_{\pi^*}}{d\nu_\pi}\big)^2\bigg]^{1/2}$.
    \item $K$ is the total number of iterations. Similarly, $K_{\rm TD}$ is the total number of TD iterations.
    \item $c_3>0$ is a constant as to quantify the difference in risk-neutral objective between optimal policy and any policy. 
\end{itemize}

\section{Algorithm Details}
\label{appendix:algorithm_details}
We provide a comparison between MVPI and STOPS.
\begin{algorithm}[ht]
\caption{\label{alg:compare}A comparison between STOPS and MVPI}
\SetKwBlock{DummyBlock}{}{}
\For{$k = 1, \dots, K$}
{
    \textbf{Step 1}: $y_k:=(1 - \gamma) J(\pi_k)$\;
    \textbf{Step 2}: $\tilde{J}(\pi_{\theta_k}) := \mathbb{E}_{(s,a)\sim\sigma_{\pi_\theta}}(r_{s,a}-\lambda r^2_{s,a} + 2\lambda r_{s,a} y_k)$;
    \SetAlgoNoLine\DummyBlock{\SetAlgoLined\uIf{\textbf{MVPI}:} {
        $\theta_{k}:=\arg\max_\theta ( \tilde{J}(\pi_{\theta_k}) )$\; 
        \tcp{This is achieved by line~\ref{line:mvpi start} to~\ref{line:mvpi end} in Algorithm~\ref{alg:MVPI}}
    }\uElseIf{\textbf{STOPS}:}
        {\uIf{select NPG update}{
            update $\theta_k$ according to Eq.~\eqref{eq:neural npg update} \;
        }\uElseIf{select PPO update}{
            update $\theta_k$ according to Eq.~\eqref{eq:neural ppo update} \;
            }
          
        }
    }
}
\textbf{Output}: $\pi_{\theta_K}$\;
\end{algorithm}
Note that neither NPG nor PPO solve $\theta_{k}:=\arg\max_\theta ( \tilde{J}(\pi_{\theta_k}) )$ directly, but instead solve an approximation optimization problem at each iteration.
We provide pseudo-code for the implementation of MVPI and VARAC in Algorithm~\ref{alg:MVPI} and \ref{alg:VARAC}.
\begin{algorithm}[ht]
\caption{\label{alg:MVPI}\textbf{MVPI with over-parameterized networks}}
\textbf{Input}: number of iteration $K$,
learning rate for natural policy gradient (resp. PPO) TD $\eta_{\rm NPG}$ (resp. $\eta_{\rm PPO}$), temperature parameters
$\{\tau_k\}^K_{k=1}$\;
\textbf{Initialization}: Initialize policy network $f((s,a);\theta,b)$ as defined in Eq.~\eqref{eq:Theta}. Set $\tau_1 = 1$.
Initialize Q-network with $(b, \omega_1)$ similarly\;
\For {$k = 1,\cdots, K$}
{
Sample a batch of transitions $\{s_t,a_t,r_t,s'_t\}^{T}_{t=1}$ following current policy with size of $T$\;
$y = \frac{1}{T} \sum_{t=1}^{T}r_t $ \;
\For{$t = 1,\cdots, T$}
{
    $\Tilde{r}_t = r_t - \lambda r_t^2 + 2\lambda r_t y$, $a'_t \sim \pi (a|s'_t)$\;
}

\Repeat{CONVERGE}{
\label{line:mvpi start}
\textbf{Q-value update}: update $\omega_k$ according to Eq.~\eqref{eq:aug TD}\;
\uIf{select NPG update}{
update $\theta_k$ according to Eq.~\eqref{eq:neural npg update}\;
}\uElseIf{select PPO update}{
update $\theta_k$ according to Eq.~\eqref{eq:neural ppo update}\; 
}}
\label{line:mvpi end}
}
\textbf{Output}: $\pi_{\theta_K}$\;
\end{algorithm}


\begin{algorithm}[ht]
\caption{\label{alg:VARAC}\textbf{VARAC}} 
\textbf{Input}: number of iteration $K$, learning rate for natural policy gradient (resp. PPO) TD $\eta_{\rm NPG}$ (resp. $\eta_{\rm PPO}$), temperature parameters
$\{\tau_k\}^K_{k=1}$\;
\textbf{Initialization}: Initialize policy network $f((s,a);\theta,b)$ as defined in Eq.~\eqref{eq:Theta}. Set $\tau_1 = 1$.
Initialize Q-network with $(b, \omega_1)$ similarly\;
\For {$k = 1,\cdots, K$}
{
Sample a batch of transitions $\{s_t,a_t,r_t,s'_t\}^{T}_{t=1}$ following current policy with size of $T$\;
$y = \frac{1}{T} \sum_{t=1}^{T}r_t $\; 
\textbf{Q-value update}: update both networks' $\omega_k$ according to Eq.~\eqref{eq:aug TD}\; 
Output $Q_{k}$ and $W_{k}$\;
update $\theta_k$ with NPG or PPO\;
}
\textbf{Output}: $\pi_{\theta_K}$\;
\end{algorithm}





\section{Theoretical Analysis Details}
\label{appendix:theory_details}
In this section, we discuss the theoretical analysis in detail. We first present the overview in Section~\ref{appendix:theory_overview}. Then we provide additional assumptions in Section~\ref{appendix:assumption}. In the rest of the section, we present all the supporting lemmas and the proof for Theorem~\ref{th:major result PPO} and~\ref{th:major result}.

\subsection{Overview}
\label{appendix:theory_overview}
\begin{figure}[htb!]
  \centering
  \includegraphics[width=0.6\linewidth]{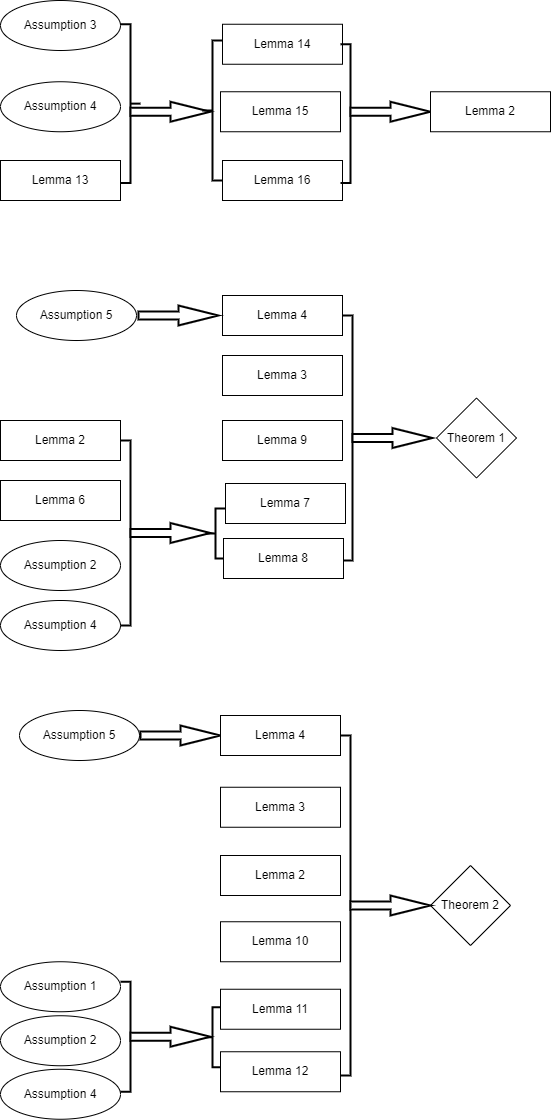}
  \caption{A flow chart of the theoretical analysis}\label{fig:lemma-theorem-relation}
\end{figure}

We provide Figure~\ref{fig:lemma-theorem-relation} to illustrate the structure of the theoretical analysis. First, under Assumption~\ref{assu: action-value function class} and~\ref{assu: state distribution regularity}, as well as Lemma~\ref{lem:projection error}. we can obtain Lemma~\ref{lem:linearization error},~\ref{lem:linearization gradient error} and~\ref{lem:variance of SUV}. These are the building blocks of lemma~\ref{lem: policy evaluation error}, which is a shared component in the analysis of both NPG and PPO. The shared components also include Lemma~\ref{lem:performance difference}, as well as Lemma~\ref{lem:y objective} obtained under Assumption~\ref{assu:estimation errors}.
For PPO analysis, under Assumption~\ref{assu: concentrability regularity} and~\ref{assu: state distribution regularity}, we obtain Lemma~\ref{lem: error propagation} and~\ref{lem:stepwise energy difference} from Lemma~\ref{lem: policy evaluation error} 
and ~\ref{lem: neural ppo update}, Then combined with Lemma~\ref{lem:performance difference}, ~\ref{lem:y objective} and~\ref{lem:ppo one step pi}, we obtain Theorem~\ref{th:major result PPO}, the major result of PPO anaysis.
Likely for NPG analysis, we first obtain Lemma~\ref{lem:npg one step pi} and~\ref{lem:Hk} under Assumption~\ref{assu: variance regularity}, ~\ref{assu: concentrability regularity} and~\ref{assu: state distribution regularity}. Then together with Lemma~\ref{lem: policy evaluation error},~\ref{lem:performance difference}, ~\ref{lem:y objective} and~\ref{lem:neural npg gradient and fisher}, we obtain Theorem~\ref{th:major result}, the major result of NPG anaysis.
\subsection{Additional Assumptions}
\label{appendix:assumption}
\begin{assumption}
\label{assu: action-value function class}
(Action-value function class) We define
\begin{align}
    &\mathcal{F}_{\Upsilon,\infty} := \Bigg\{f(s,a;\theta)=f_0(s,a))+\nonumber \\
    &\int \mathbbm{1}\{\theta^\top(s,a)>0\} (s,a)^\top\iota(\theta)d\mu (w):\Vert\iota(\theta)\Vert_\infty \leq \Upsilon/\sqrt{d}\Bigg\}
\end{align}
Where $\mu:\mathbb{R}^d \rightarrow [0,1] $ is a probability density function of $\mathcal{N}(0,I_d/d)$. $f_0(s,a)$ is the two-layer neural network corresponding to the initial parameter $\Theta_{{\rm init}}$, and $\iota:\mathbb{R}^d \rightarrow \mathbb{R}^d $ is a weighted function. 
We assume that $\Tilde{Q}_\pi\in\mathcal{F}_{\Upsilon,\infty}$ for all $\pi$.
\end{assumption}
\begin{assumption}
\label{assu: state distribution regularity}
(Regularity of stationary distribution) For any policy $\pi$, and  $\forall x \in \mathbb{R}^d, \forall\Vert x\Vert_2=1$, and $\forall u>0$, we assume that there exists a constant $c > 0$ such that 
$
    \mathbb{E}_{(s,a) \sim \sigma_\pi}\big[\mathbbm{1}\{|x^\top(s,a)|\leq u\}\big]\leq c u.
$
\end{assumption}
Assumption~\ref{assu: action-value function class} is a mild regularity condition on $Q_\pi$, as $\mathcal{F}_{\Upsilon,\infty}$ is a sufficiently rich function class and approximates a subset of the reproducing kernel Hilbert space (RKHS)~\cite{rahimi2008weighted}. Similar assumptions are widely imposed~\cite{munos2008finitetime,antos2007fitted,farahmand2016regularized,yang2020reinforcement,wang2019neural}. Assumption~\ref{assu: state distribution regularity} is a regularity condition on the transition kernel $\mathcal{P}$. Such regularity holds so long as $\sigma_\pi$ has an upper bound density, satisfying most Markov chains.

In \cite{zhong2020risksensitive} Lemma 4.15, they make a mistake in the proof. They accidentally flip a sign in $y^*-\bar{y}$ when transitioning from the first equation in the proof to Eq.(4.15). This invalidates the conclusion in Eq.(4.17), an essential part of the proof. We tackle this issue by proposing the next assumption.
\begin{assumption}
\label{assu:estimation errors}
(Convergence Rate of $J(\pi)$) We assume  $\pi^*$ 
(the optimal policy to the risk-averse objective function $J_\lambda(\pi)$) converges to the risk-neutral objective $J(\pi)$ for both NPG and PPO with the over-parameterized neural network to be $\mathcal{O}(1/\sqrt{k})$. Specifically, there exists a constant $c_3>0$ such that,
\begin{align}
    J(\pi^*) - J(\pi_k) \leq \frac{c_3}{\sqrt{k}}
\end{align}
\end{assumption}
It was proved~\cite{wang2019neural,liu2019neural} that the optimal policy w.r.t the risk-neutral objective $J(\pi)$ obtained by NPG and PPO method with the over-parameterized two-layer neural network converges to the globally optimal policy at a rate of $\mathcal{O}(1/\sqrt{K})$, where $K$ is the number of iteration. Since our method uses similar settings, we assume the convergence rates of risk-neutral objective $J(\pi)$ in our paper follow their results.

In the following subsections, we study STOPS's convergence of global optimality and provide a proof sketch.

\subsection{Proof of Theorem~\ref{th:major result PPO}}
\label{sec:proof of ppo}
We first present the analysis of policy evaluation error, which is induced by TD update in Line~\ref{line:critic update} of Algorithm~\ref{alg:TOPS}. We characterize the policy evaluation error in the following lemma:
\begin{lemma}
\label{lem: policy evaluation error}
(Policy Evaluation Error) We set learning rate of TD $\eta_{\text{TD}} = \min\{(1-\gamma)/3(1+\gamma)^2, 1/\sqrt{K_{{\rm TD}}}\}$. Under Assumption~\ref{assu: action-value function class} and~\ref{assu: state distribution regularity}, it holds that, with probability of $1-\delta$,
\begin{align}
    &\Vert\Tilde{Q}_{\omega_k}-\Tilde{Q}_{\pi_k}\Vert^2_{\nu_{\pi_k}} \nonumber\\
    =& \mathcal{O}(\Upsilon^{3}m^{-1/2}\log(1/\delta)+\Upsilon^{5/2}m^{-1/4}\sqrt{\log(1/\delta)}\nonumber \\
    &+\Upsilon r_{\max}^2m^{-1/4}+\Upsilon^2K_{{\rm TD}}^{-1/2}+\Upsilon), \label{eq:policy evaluation error}
\end{align} 
\end{lemma}
where $\Tilde{Q}_{\pi_k}$ is the Q-value function of the augmented MDP, and $\Tilde{Q}_{\omega_k}$ is its estimator at the $k$-th iteration.
We provide the proof and its supporting lemmas in Appendix~\ref{sec:proof of policy evaluation error}. In the following, we establish the error induced by the policy update. 
Eq.~\eqref{eq:volatility new MDP reward} can be re-expressed as
\begin{align}
    J_\lambda^y(\pi) &= \sum_{s, a} \sigma_\pi \big(r_{s,a} -\lambda r_{s,a}^2 + 2\lambda r_{s,a}{y_{k + 1}}\big) - \lambda y_{k+1}^2 \label{eq:mean-volitality subproblem}
\end{align}

It can be shown that $\forall \pi, \max_{y}J_\lambda^y (\pi) =   J_\lambda(\pi) $~\cite{xie2018block,zhang2020meanvariance}. We denote the optimal policy to the augmented MDP associated with $y^*$ by $\pi^*(y^*)$. By definition, it is obvious that $\pi^*$ and $\pi^*(y^*)$ are equivalent. For simplicity, we will use the unified term $\pi^*$ in the rest of the paper. We present Lemma~\ref{lem:performance difference} and ~\ref{lem:y objective}.


\begin{lemma}
\label{lem:performance difference}
(Policy's Performance Difference) For reward-volatility objective w.r.t. auxiliary variable $y$ as $J^y_\lambda (\pi)$ defined in Eq.~\eqref{eq:mean-volitality subproblem}. For any policy $\pi$ and $\pi'$, we have the following,
\begin{align}
    J^y_\lambda (\pi') - J^y_\lambda (\pi) =& (1-\gamma)^{-1}\mathbb{E}_{s \sim \nu_{\pi'}}\big[\mathbb{E}_{a \sim \pi'}[\Tilde{Q}_{\pi,y} ]\\
     &-\mathbb{E}_{a \sim \pi}[\Tilde{Q}_{\pi,y} ]\big],
    \label{eq:lem2}
\end{align}
where $\Tilde{Q}_{\pi,y} $ is the state-action value function of the augmented MDP, and its rewards are associated with $y$.
\end{lemma}
\begin{proof}
When $y$ is fixed, 
\begin{align}
    & J^y_\lambda (\pi') - J^y_\lambda (\pi) \nonumber\\
    =& \sum_{s, a} \sigma_{\pi'}\Tilde{r}_{s,a}  - \sum_{s, a} \sigma_\pi \Tilde{r}_{s,a} = \Tilde{J} (\pi')) - \Tilde{J} (\pi) \\\label{eq:lemma1a}
\end{align}
We then follow Lemma 6.1 in \cite{kakade2002approximate}:
\begin{align}
    \Tilde{J} (\pi') - \Tilde{J} (\pi) = (1-\gamma)^{-1}\mathbb{E}_{(s,a) \sim \sigma_{\pi'}}\left[\Tilde{A}_\pi\right] \label{eq:lemma1b}
\end{align}
where $\Tilde{A}_\pi = \Tilde{Q}_\pi - \Tilde{V}_\pi$ is the advantage function of policy $\pi$. Meanwhile, 
\begin{align}
    \mathbb{E}_{a \sim \pi'}[\Tilde{A}_\pi] &= \mathbb{E}_{a \sim \pi'}[\Tilde{Q}_\pi] - \Tilde{V}_\pi = \mathbb{E}_{a \sim \pi'}[\Tilde{Q}_\pi] - \mathbb{E}_{a \sim \pi}[\Tilde{Q}_\pi] \label{eq:lemma1c}
\end{align}
From Eq.~\eqref{eq:lemma1a}, Eq.~\eqref{eq:lemma1b} and Eq.~\eqref{eq:lemma1c}, we complete the proof.
\end{proof}
%
Lemma~\ref{lem:performance difference} is inspired by~\cite{kakade2002approximate} and adopted by most work on global convergence~\cite{agarwal2020theory, liu2019neural, xu2021crpo}. Next, we derive an upper bound for the error of the critic update in Line~\ref{line:y update} of Algorithm~\ref{alg:TOPS}:
\begin{lemma}
\label{lem:y objective}
($y$ Update Error) We characterize the error induced by the estimation of auxiliary variable y w.r.t the optimal value $y^*$ at $k$-th iteration as, 
$
     J^{y^*}_\lambda(\pi^*)-J^{\hat{y}_k}_\lambda(\pi^*) = \frac{2c_3 r_{\max}(1-\gamma)\lambda}{\sqrt{k}}, 
$
where $r_{\max}$ is the bound of the original reward, and $c_3$ is a constant error term.
\end{lemma}
\begin{proof}
We start from the subproblem objective defined in Eq.~\eqref{eq:mean-volitality subproblem} with $y^*$ and $\hat{y}_k$:
\begin{align}
    &J^{y^*}_\lambda(\pi^*)-J^{\hat{y}_k}_\lambda(\pi^*) \nonumber \\
    =& \bigg(\sum_{s, a} \sigma_{\pi^*} \big(r_{s,a}-\lambda r^2_{s,a}+ 2\lambda r_{s,a}{y^*}\big) - \lambda y^*{}^2 \bigg)-\nonumber \\
    &  \bigg(\sum_{s, a} \sigma_{\pi^*} \big(r_{s,a}-\lambda r^2_{s,a}+ 2\lambda r_{s,a}{\hat{y}_k}\big) - \lambda \hat{y}_k^2\bigg)\nonumber \\
    =& 2\lambda\big(\sum_{s,a}\sigma_{\pi^*}r_{s,a}\big)(y^*-\hat{y}_k) - \lambda(y^*{}^2-\hat{y}_k^2) \nonumber \\
    =& \lambda\langle y^*-\hat{y}_k, 2(1-\gamma)J(\pi^*)-y^*-\hat{y}_k\rangle \\
    =& (1-\gamma)\lambda \langle y^*-\hat{y}_k, J(\pi^*)-\hat{J}(\pi_k)\rangle
\end{align}
where we obtain the final two equalities by the definition of $J_\pi$ and $y$. Because $r_{s,a}$ is upper-bounded by a constant $r_{\max}$, we have $|y^*-\hat{y}_k | \leq 2r_{\max}$. Under Assumption~\ref{assu:estimation errors} we have,
\begin{align}
     J^{y^*}_\lambda(\pi^*)-J^{\hat{y}_k}_\lambda(\pi^*) = \frac{2 c_3 r_{\max}(1-\gamma)\lambda}{\sqrt{k}}
\end{align}
Thus we finish the proof.
\end{proof}
From Lemma~\ref{lem:performance difference} and~\ref{lem:y objective}, we can also obtain the following Lemma.
\begin{lemma}
\label{lem:performance difference pi and y}
(Performance Difference on $\pi$ and $y$) For reward-volatility objective w.r.t. auxiliary variable $y$ as $J^y_\lambda (\pi)$ defined in Eq.~\eqref{eq:mean-volitality subproblem}. For any $\pi,y$ and the optimal $\pi*,y*$, we have the following,
\begin{align}
    J^{y^*}_\lambda (\pi^*) - J^y_\lambda (\pi) =& (1-\gamma)^{-1}\mathbb{E}_{s \sim \nu_{\pi^*}}\big[\mathbb{E}_{a \sim \pi^*}[\Tilde{Q}_{\pi,y} ]\\
     &-\mathbb{E}_{a \sim \pi}[\Tilde{Q}_{\pi,y} ]\big] + \frac{2 c_3 r_{\max}(1-\gamma)\lambda}{\sqrt{k}}.
\end{align}
where $\Tilde{Q}_{\pi,y} $ is the state-action value function of the augmented MDP, and its rewards are associated with $y$.
\end{lemma}
\begin{proof}
It is easy to see that $J^{y^*}_\lambda (\pi^*) - J^y_\lambda (\pi) = J^{y^*}_\lambda (\pi^*) - J^y_\lambda (\pi^*) + J^y_\lambda (\pi^*) - J^{y}_\lambda (\pi)$. Then replace $J^y_\lambda (\pi^*) - J^{y}_\lambda (\pi)$ with Lemma~\ref{lem:performance difference} and $J^{y^*}_\lambda (\pi^*) - J^y_\lambda (\pi^*)$ with Lemma~\ref{lem:y objective}, we finish the proof.
\end{proof}
Lemma~\ref{lem:performance difference pi and y} quantifies the performance difference of $J^{y}_\lambda (\pi)$ between any pair $\pi,y$ and the optimal $\pi*,y*$, while Lemma~\ref{lem:performance difference} only quantifies the performance difference of $J^{y}_\lambda (\pi)$ between $\pi$ and $\pi'$ when $y$ is fixed.

We now study the global convergence of STOPS with neural PPO as the policy update component. First, we define the neural PPO update rule.
\begin{lemma}
\label{lem: neural ppo update}
 \cite{liu2019neural}.Let $\pi_{\theta_k} \varpropto \exp\{\tau^{-1}_k f_{\theta_k}\}$ be an energy-based policy. We define the update 
 $$\hat{\pi}_{k+1} = \arg{\max_\pi}\mathbb{E}_{s\sim\nu_k}[\mathbb{E}_{\pi}[Q_{\omega_k}] - \beta_k \textrm{KL}(\pi_\theta\Vert\pi_{\theta_k})]$$,
 where $Q_{\omega_k}$ is the estimator of the exact action-value function $Q^{\pi_{\theta_k}}$. We have
 \begin{align}
     \hat{\pi}_{k+1} \varpropto \exp\{\beta^{-1}_k Q_{\omega_k} + \tau^{-1}_k f_{\theta_k}\}
 \end{align}
 And to represent $\hat{\pi}_{k+1}$ with $\pi_{\theta_{k+1}} \varpropto \exp\{\tau^{-1}_{k+1} f_{\theta_{k+1}}\}$, we solve the following subproblem,
 \begin{align}
     \theta_{k+1} =& \arg{\min_{\theta\in\mathbb{D}}}\mathbb{E}_{(s,a)\sim\sigma_k}[(f_\theta(s,a)- \tau_{k+1}(\beta^{-1}_k Q_{\omega_k}(s,a)\nonumber \\
     &+ \tau^{-1}_k f_{\theta_k}(s,a)))^2]
 \end{align}
\end{lemma}
We analyze the policy improvement error in Line~\ref{line:ppo update} of Algorithm~\ref{alg:TOPS}.
\cite{liu2019neural} proves that the policy improvement error can be characterized similarly to the policy evaluation error as in Eq.~\eqref{eq:policy evaluation error}. Recall $\Tilde{Q}_{\omega_k}$ is the estimator of Q-value, $f_{\theta_k}$ the energy function for policy, and $f_{\hat{\theta}}$ its estimator. We characterize the policy improvement error as follows: Under Assumptions~\ref{assu: action-value function class} and~\ref{assu: state distribution regularity}, we set the learning rate of PPO $\eta_{{\rm PPO}}=\min\{(1-\gamma)/3(1+\gamma)^2 1/\sqrt{K_{{\rm TD}}}\}$, and with a probability of $1-\delta$:
\begin{align}
    &\Vert(f_{\hat{\theta}} - \tau_{k+1}(\beta^{-1} \Tilde{Q}_{\omega_k} +\tau^{-1}_k f_{\theta_k})\Vert^2_{} \nonumber\\
    =& \mathcal{O}(\Upsilon^{3}m^{-1/2}\log(1/\delta)+\Upsilon^{5/2}m^{-1/4}\sqrt{\log(1/\delta)}\nonumber \\
    &+\Upsilon r_{\max}^2m^{-1/4}+\Upsilon^2K_{{\rm TD}}^{-1/2}+\Upsilon). \label{eq:policy improvement error}
\end{align} 
We quantify how the errors propagate in neural PPO~\cite{liu2019neural} in the following.
\begin{lemma}
\label{lem: error propagation}
\cite{liu2019neural}.(Error Propagation) 
We have,
\begin{align}
    &\big|\mathbb{E}_{s \sim \nu_{\pi^*}}\big[\mathbb{E}_{a \sim \pi^*}[\log(\pi_{\theta_{k+1}}/\pi_{k+1}]-\mathbb{E}_{a \sim \pi_{\theta_k}} \nonumber \\
    &[\log(\pi_{\theta_{k+1}}/\pi_{k+1}]\big]\big|\leq \tau^{-1}_{k+1}\varepsilon''_{k}\varphi^*_{k+1} + \beta^{-1}\varepsilon''_{k}\psi^*_{k}\\ \label{eq:error propagation}
\end{align}
\end{lemma}
$\varepsilon''_{k}$ are defined in Eq.\eqref{eq:policy evaluation error} as well as Eq.\eqref{eq:policy improvement error}. $\varphi^*_{k} = \mathbb{E}_{(s,a) \sim \sigma_\pi}\bigg[\big(\frac{d\pi^*}{d\pi_0}-\frac{d\pi_{\theta_k}}{d\pi_0}\big)^2\bigg]^{1/2}, \psi^*_{k} = \mathbb{E}_{(s,a) \sim \sigma_\pi}\bigg[\big(\frac{d\sigma_{\pi^*}}{d\sigma_\pi}-\frac{d\nu_{\pi^*}}{d\nu_\pi}\big)^2\bigg]^{1/2}$.
$\frac{d\pi^*}{d\pi_0},\frac{d\pi_{\theta_k}}{d\pi_0},\frac{d\sigma_{\pi^*}}{d\sigma_\pi},\frac{d\nu_{\pi^*}}{d\nu_\pi}$ are the Radon-Nikodym derivatives~\cite{Konstantopoulos2011}. 
We denote RHS in Eq.~\eqref{eq:error propagation} by $\varepsilon_k = \tau^{-1}_{k+1}\varepsilon''_{k}\varphi^*_{k+1} + \beta^{-1}\varepsilon''_{k}\psi^*_{k}$.
Lemma~\ref{lem: error propagation} essentially quantifies the error from which we use the two-layer neural network to approximate the action-value function and policy instead of having access to the exact ones. Please refer to \cite{liu2019neural} for complete proofs of Lemma~\ref{lem: neural ppo update} and ~\ref{lem: error propagation}.
\begin{align}
    &\big|\mathbb{E}_{s \sim \nu_{\pi^*}}\big[\mathbb{E}_{a \sim \pi^*}[\log(\pi_{\theta_{k+1}}/\pi_{k+1}] - \nonumber \\
    &\mathbb{E}_{a \sim \pi_{\theta_k}}[\log(\pi_{\theta_{k+1}}/\pi_{k+1}]\big]\big| \leq \tau^{-1}_{k+1}\varepsilon''_{k}\varphi^*_{k+1} + \beta^{-1}\varepsilon''_{k}\psi^*_{k} 
\end{align}

We then characterize the difference between energy functions in each step~\cite{liu2019neural}. Under the optimal policy $\pi*$,
\begin{lemma}
\label{lem:stepwise energy difference}
\cite{liu2019neural}.(Stepwise Energy Function difference) Under the same condition of Lemma~\ref{lem: error propagation}, we have
\begin{align}
    \mathbb{E}_{s \sim \nu_{\pi^*}}[\Vert\tau^{-1}_{k+1} f_{\theta_{k+1}}-\tau^{-1}_{k} f_{\theta_{k}}\Vert^2_\infty]
    \leq 2\varepsilon'_k+2\beta^{-2}_k U, \\\label{eq:stepwise energy difference}
\end{align}
where $\varepsilon'_k = |\mathcal{A}|\tau^{-2}_{k+1}\epsilon^2_{k+1}$ \\
and $U = 2\mathbb{E}_{s \sim \nu_{\pi^*}}[\max_{a\in\mathcal{A}}(\Tilde{Q}_{\omega_{0}})^2] + 2\Upsilon^2$. 
\end{lemma}
\begin{proof}
By the triangle inequality, we get the following,
\begin{align}
    &\Vert\tau^{-1}_{k+1} f_{\theta_{k+1}}-\tau^{-1}_{k} f_{\theta_{k}}\Vert^2_\infty \\
    \leq& 2\big( \Vert\tau^{-1}_{k+1} f_{\theta_{k+1}}-\tau^{-1}_{k} f_{\theta_{k}}-\beta^{-1}\Tilde{Q}_{\omega_k}\Vert^2_\infty + \Vert\beta^{-1}\Tilde{Q}_{\omega_k}\Vert^2_\infty\big) \label{eq:stepwise energy 1}
\end{align}
We take the expectation of both sides of Eq.~\eqref{eq:stepwise energy 1} with respect to $s\sim \nu_{\pi^*}$. With the 1-Lipshitz continuity of $\Tilde{Q}_{\omega_k}$ in $\omega$ and $\Vert\omega_k-\Theta_{\rm init}\Vert_2 \leq \Upsilon$, we have,
\begin{align}
    &\mathbb{E}_{\nu_{\pi^*}}\big[\Vert\tau^{-1}_{k+1} f_{\theta_{k+1}}-\tau^{-1}_{k} f_{\theta_{k}}\Vert^2_\infty\big] \\
    \leq& 2(|\mathcal{A}|\tau^{-2}_{k+1}\epsilon^2_{k+1} + \mathbb{E}_{s \sim \nu_{\pi^*}}[\max_{a\in\mathcal{A}}(\Tilde{Q}_{\omega_{0}})^2] + \Upsilon^2)
\end{align}
Thus complete the proof.
\end{proof}
We then derive a difference term associated with $\pi_{k+1}$ and $\pi_{\theta_k}$,
where at the $k$-th iteration $\pi_{k+1}$ is the solution for the following subproblem,
\begin{align}
    \pi_{k+1}=\arg{\max_\pi}\Big(\mathbb{E}_{s \sim \nu_{\pi_k}}\big[\mathbb{E}_{a \sim \pi}[\Tilde{Q}_{\pi_k,\hat{y}_k}]-\beta \text{KL}(\pi\Vert\pi_{\theta_k})\big]\Big)
\end{align} 
and $\pi_{\theta_k}$ is the policy parameterized by the two-layered over-parameterized neural network.
The following lemma establishes the one-step descent of the KL-divergence in the policy space: 
\begin{lemma}
\label{lem:ppo one step pi}
(One-step difference of $\pi$) For $\pi_{k+1}$ and $\pi_{\theta_k}$, we have
\begin{align}
    & \text{KL}(\pi^*\Vert\pi_{\theta_{k}})-\text{KL}(\pi^*\Vert\pi_{\theta_{k+1}})\nonumber\\
    \geq & \big(\mathbb{E}_{a \sim \pi^*}[\log(\frac{\pi_{\theta_{k+1}}}{\pi_{k+1}})]- \mathbb{E}_{a \sim \pi_{\theta_{k}}}[\log(\frac{\pi_{\theta_{k+1}}}{\pi_{k+1}})]\big)\nonumber \\
    &+ \beta^{-1}\big(\mathbb{E}_{a \sim \pi^*}[\Tilde{Q}_{\pi_k,\hat{y}_k}]-\mathbb{E}_{a \sim \pi_{\theta_{k}}}[\Tilde{Q}_{\pi_k,\hat{y}_k}]\big)\nonumber \\
    &+ \frac{1}{2}\Vert\pi_{\theta_{k+1}} - \pi_{\theta_{k}}\Vert^2_1 + \big(\mathbb{E}_{a \sim \pi_{\theta_{k}}}[\tau^{-1}_{k+1} f_{\theta_{k+1}}-\tau^{-1}_{k} f_{\theta_{k}}] \nonumber \\
    &-\mathbb{E}_{a \sim \pi_{\theta_{k+1}}}[\tau^{-1}_{k+1} f_{\theta_{k+1}}-\tau^{-1}_{k} f_{\theta_{k}}]\big)\label{eq:one step pi result}
\end{align}
\end{lemma}

\begin{proof}
We start from
\begin{align}
    & \text{KL}(\pi^*\Vert\pi_{\theta_{k}})-\text{KL}(\pi^*\Vert\pi_{\theta_{k+1}}) = \mathbb{E}_{a \sim \pi^*}[\log(\frac{\pi_{\theta_{k+1}}}{\pi_{\theta_{k}}})] \nonumber\\
    \text{(B}&\text{y definition, }\text{KL}(\pi_{\theta_{k+1}}\Vert\pi_{\theta_{k}})= \mathbb{E}_{a \sim \pi_{\theta_{k+1}}}[\log(\frac{\pi_{\theta_{k+1}}}{\pi_{\theta_{k}}})]\big)) \nonumber\\
    =& \big(\mathbb{E}_{a \sim \pi^*}[\log(\frac{\pi_{\theta_{k+1}}}{\pi_{\theta_{k}}})]-\mathbb{E}_{a \sim \pi_{\theta_{k+1}}}[\log(\frac{\pi_{\theta_{k+1}}}{\pi_{\theta_{k}}})]\big) + \nonumber\\
    &\text{KL}(\pi_{\theta_{k+1}}\Vert\pi_{\theta_{k}}) \nonumber\\
    \text{W}&\text{e then add and subtract terms, }\nonumber\\
    =&\mathbb{E}_{a \sim \pi^*}[\log(\frac{\pi_{\theta_{k+1}}}{\pi_{\theta_{k}}})]-\mathbb{E}_{a \sim \pi_{\theta_{k+1}}}[\log(\frac{\pi_{\theta_{k+1}}}{\pi_{\theta_{k}}})] + \text{KL}\nonumber\\
    &(\pi_{\theta_{k+1}}\Vert\pi_{\theta_{k}})+ \beta^{-1}\big(\mathbb{E}_{a \sim \pi^*}[\Tilde{Q}_{\pi_k,\hat{y}_k}]-\mathbb{E}_{a \sim \pi_{\theta_{k}}}[\Tilde{Q}_{\pi_k,\hat{y}_k}]\big) \nonumber\\
    &- \beta^{-1}\big(\mathbb{E}_{a \sim \pi^*}[\Tilde{Q}_{\pi_k,\hat{y}_k}]-\mathbb{E}_{a \sim \pi_{\theta_{k}}}[\Tilde{Q}_{\pi_k,\hat{y}_k}]\big) \nonumber\\
    &+ \mathbb{E}_{a \sim \pi_{\theta_{k}}}[\log(\frac{\pi_{\theta_{k+1}}}{\pi_{\theta_{k}}})]-\mathbb{E}_{a \sim \pi_{\theta_{k}}}[\log({\pi_{\theta_{k+1}}}{\pi_{\theta_{k}}})]\\
      &\text{Rearrange the terms and we get, }\nonumber\\
    =& \big(\mathbb{E}_{a \sim \pi^*}[\log(\pi_{\theta_{k+1}})-\log(\pi_{\theta_{k}})-\beta^{-1} \Tilde{Q}_{\pi_k,\hat{y}_k}]\nonumber\\
    &- \mathbb{E}_{a \sim \pi_{\theta_{k}}}[\log(\pi_{\theta_{k+1}})-\log(\pi_{\theta_{k}})-\beta^{-1} \Tilde{Q}_{\pi_k,\hat{y}_k}]\big)\nonumber\\
    &+ \beta^{-1}\big(\mathbb{E}_{a \sim \pi^*}[\Tilde{Q}_{\pi_k,\hat{y}_k}]-\mathbb{E}_{a \sim \pi_{\theta_{k}}}[\Tilde{Q}_{\pi_k,\hat{y}_k}]\big)+ \text{KL}\nonumber\\
    &(\pi_{\theta_{k+1}}\Vert\pi_{\theta_{k}}) + \big(\mathbb{E}_{a \sim \pi_{\theta_{k}}}[\log(\frac{\pi_{\theta_{k+1}}}{\pi_{\theta_{k}}})]-\mathbb{E}_{a \sim \pi_{\theta_{k+1}}}\nonumber\\
    &[\log({\pi_{\theta_{k+1}}}{\pi_{\theta_{k}}})]\big)\label{eq:one step pi eq1}
\end{align}
Recall that $\pi_{k+1} \varpropto \exp\{\tau^{-1}_k f_{\theta_k}+\beta^{-1} \tilde{Q}^y_{\pi_k}\}$. We define the two normalization factors associated with ideal improved policy $\pi_{k+1}$ and the current parameterized policy $\pi_{\theta_k}$ as,
\begin{align}
    &Z_{k+1}(s) := \sum_{a'\in\mathcal{A}}\exp\{\tau^{-1}_k f_{\theta_k}(s,a')+\beta^{-1} \tilde{Q}^y_{\pi_k}(s,a')\}\\
    &Z_{\theta_{k+1}}(s) := \sum_{a'\in\mathcal{A}}\exp\{\tau^{-1}_{k+1} f_{\theta_{k+1}}(s,a')\}
\end{align}
We then have,
\begin{align}
    &\pi_{k+1}(a|s) = \frac{\exp\{\tau^{-1}_k f_{\theta_k}(s,a)+\beta^{-1} \tilde{Q}^y_{\pi_k}(s,a)\}}{Z_{k+1}(s)},\label{eq:ppo k softmax}\\
    &\pi_{\theta_{k+1}}(a|s) = \frac{\exp\{\tau^{-1}_{k+1} f_{\theta_{k+1}}(s,a)\}}{Z_{\theta_{k+1}}(s)} \label{eq:ppo theta softmax}
\end{align}
For any $\pi, \pi'$ and $k$, we have,
\begin{align}
    \mathbb{E}_{a \sim \pi}[\log Z_{\theta_{k+1}}]-\mathbb{E}_{a \sim \pi'}[\log Z_{\theta_{k+1}}]= 0\label{eq:z theta property}\\
    \mathbb{E}_{a \sim \pi}[\log Z_{k+1}]-\mathbb{E}_{a \sim \pi'}[\log Z_{k+1}] = 0 \label{eq:z k property}
\end{align}
Now we look back at a few terms on RHS from Eq.\eqref{eq:one step pi eq1}:
\begin{align}
    &\mathbb{E}_{a \sim \pi^*}\big[\log(\pi_{\theta_{k}})+\beta^{-1} \Tilde{Q}_{\pi_k,\hat{y}_k}\big]\\
    &- \mathbb{E}_{a \sim \pi_{\theta_{k}}}\big[\log(\pi_{\theta_{k}})+\beta^{-1} \Tilde{Q}_{\pi_k,\hat{y}_k}\big]\\
    =&\big(\mathbb{E}_{a \sim \pi^*}[\tau^{-1}_k f_{\theta_k}+\beta^{-1} \Tilde{Q}_{\pi_k,\hat{y}_k}-\log Z_{\theta_{k+1}}]\nonumber\\
    &-\mathbb{E}_{a \sim \pi_{\theta_{k}}}[\tau^{-1}_k f_{\theta_k}+\beta^{-1} \Tilde{Q}_{\pi_k,\hat{y}_k}-\log Z_{\theta_{k+1}}]\big) \nonumber\\
    =&\mathbb{E}_{a \sim \pi^*}\Big[\log\frac{\exp\{\tau^{-1}_k f_{\theta_k}+\beta^{-1} \Tilde{Q}_{\pi_k,\hat{y}_k}\}}{Z_{k+1}}\Big]\nonumber\\
    &-\mathbb{E}_{a \sim \pi_{\theta_{k}}}\Big[\log\frac{\exp\{\tau^{-1}_k f_{\theta_k}+\beta^{-1} \Tilde{Q}_{\pi_k,\hat{y}_k}\}}{Z_{k+1}}\Big]\nonumber\\
    =&\mathbb{E}_{a \sim \pi^*}[\log\pi_{k+1}]- \mathbb{E}_{a \sim \pi_{\theta_{k}}}[\log\pi_{k+1}]\label{eq:one step pi aux eq1}
\end{align}
For Eq.~\eqref{eq:one step pi aux eq1}, we obtain the first equality by Eq.~\eqref{eq:ppo theta softmax}. Then, by swapping Eq.~\eqref{eq:z theta property} with Eq.~\eqref{eq:z k property}, we obtain the second equality. We achieve the concluding step with the definition in Eq.~\eqref{eq:ppo k softmax}. Following a similar logic, we have,
\begin{align}
    &\mathbb{E}_{a \sim \pi_{\theta_{k}}}[\log(\frac{\pi_{\theta_{k+1}}}{\pi_{\theta_{k}}})-\mathbb{E}_{a \sim \pi_{\theta_{k+1}}}[\log(\frac{\pi_{\theta_{k+1}}}{\pi_{\theta_{k}}})]\nonumber\\
    =& \mathbb{E}_{a \sim \pi_{\theta_{k}}}[\tau^{-1}_{k+1} f_{\theta_{k+1}}-\log Z_{\theta_{k+1}}-\tau^{-1}_{k} f_{\theta_{k}}+\log Z_{\theta_{k}}]-\nonumber\\
    &\mathbb{E}_{a \sim \pi_{\theta_{k+1}}}[\tau^{-1}_{k+1} f_{\theta_{k+1}} - \log Z_{\theta_{k+1}} - \tau^{-1}_{k} f_{\theta_{k}}+\log Z_{\theta_{k}}] \nonumber\\
    =&\mathbb{E}_{a \sim \pi_{\theta_{k}}}[\tau^{-1}_{k+1} f_{\theta_{k+1}} - \tau^{-1}_{k} f_{\theta_{k}}] -\mathbb{E}_{a \sim \pi_{\theta_{k+1}}}[\tau^{-1}_{k+1} f_{\theta_{k+1}} - \\
     &\tau^{-1}_{k} f_{\theta_{k}}]\label{eq:one step pi aux eq2}
\end{align}
Finally, by using the Pinsker's inequality~\cite{csiszar2011information}, we have,
\begin{align}
    \text{KL}(\pi_{\theta_{k+1}}\Vert\pi_{\theta_{k}}) \geq 1/2 \Vert\pi_{\theta_{k+1}} - \pi_{\theta_{k}}\Vert^2_1 \label{eq:one step pi aux eq3}
\end{align}
Plugging Eqs.~\eqref{eq:one step pi aux eq1},~\eqref{eq:one step pi aux eq2}, and ~\eqref{eq:one step pi aux eq3} into Eq.~\eqref{eq:one step pi eq1}, we have
\begin{align}
    & \text{KL}(\pi^*\Vert\pi_{\theta_{k}})-\text{KL}(\pi^*\Vert\pi_{\theta_{k+1}})\nonumber\\
    \geq & \big(\mathbb{E}_{a \sim \pi^*}[\log(\pi_{\theta_{k+1}})-\log(\pi_{k+1})]- \mathbb{E}_{a \sim \pi_{\theta_{k}}}[\log(\pi_{\theta_{k+1}})\nonumber \\
    &-\log(\pi_{k+1})]\big)+ \beta^{-1}\big(\mathbb{E}_{a \sim \pi^*}[\Tilde{Q}_{\pi_k,\hat{y}_k}]-\mathbb{E}_{a \sim \pi_{\theta_{k}}}[\Tilde{Q}_{\pi_k,\hat{y}_k}]\big)\nonumber \\
    &+ \frac{1}{2}\Vert\pi_{\theta_{k+1}} - \pi_{\theta_{k}}\Vert^2_1 + \big(\mathbb{E}_{a \sim \pi_{\theta_{k}}}[\tau^{-1}_{k+1} f_{\theta_{k+1}}-\tau^{-1}_{k} f_{\theta_{k}}]\nonumber \\
    &- \mathbb{E}_{a \sim \pi_{\theta_{k+1}}}[\tau^{-1}_{k+1} f_{\theta_{k+1}}-\tau^{-1}_{k} f_{\theta_{k}}]\big) 
\end{align}
Rearranging the terms, we obtain Lemma~\ref{lem:ppo one step pi}.
\end{proof}

Lemma~\ref{lem:ppo one step pi} serves as an intermediate-term for the major result's proof. We obtain upper bounds by telescoping this term in Theorem~\ref{th:major result PPO}. Now we are ready to present the proof for Theorem~\ref{th:major result PPO}.

\begin{proof}
First we take expectation of both sides of Eq.~\eqref{eq:one step pi result} with respect to $s\sim \nu_{\pi^*}$ from Lemma~\ref{lem:ppo one step pi} and insert  Eq~\eqref{eq:error propagation} to obtain,
\begin{align}
    &\mathbb{E}_{s \sim \nu_{\pi^*}}[\text{KL}(\pi^*\Vert\pi_{\theta_{k+1}})] - \mathbb{E}_{s \sim \nu_{\pi^*}}[\text{KL}(\pi^*\Vert\pi_{\theta_{k}})] \nonumber\\
    \leq& \varepsilon_k - \beta^{-1}\mathbb{E}_{s \sim \nu_{\pi^*}}\big[\mathbb{E}_{a \sim \pi^*}[\Tilde{Q}_{\pi_k,\hat{y}_k}]-\mathbb{E}_{a \sim \pi_{\theta_{k}}}[\Tilde{Q}_{\pi_k,\hat{y}_k}]\big] \nonumber\\
    -& 1/2\mathbb{E}_{s \sim \nu_{\pi^*}}\big[\Vert\pi_{\theta_{k+1}} - \pi_{\theta_{k}}\Vert^2_1\big] 
    - \mathbb{E}_{s \sim \nu_{\pi^*}}\big[\mathbb{E}_{a \sim \pi_{\theta_{k}}}\label{eq:ppo therorem 1st step}\\
    [&\tau^{-1}_{k+1} f_{\theta_{k+1}}- \tau^{-1}_{k} f_{\theta_{k}}]- \mathbb{E}_{a \sim \pi_{\theta_{k+1}}}[\tau^{-1}_{k+1} f_{\theta_{k+1}}
    -\tau^{-1}_{k} f_{\theta_{k}}]\big] 
\end{align}
Then, by Lemma~\ref{lem:performance difference}, we have,
\begin{align}
    &\beta^{-1}\mathbb{E}_{s \sim \nu_{\pi^*}}\big[\mathbb{E}_{a \sim \pi^*}[\Tilde{Q}_{\pi_k,\hat{y}_k}]-\mathbb{E}_{a \sim \pi_{\theta_{k}}}[\Tilde{Q}_{\pi_k,\hat{y}_k}]\big]\\
    =&\beta^{-1}(1-\gamma)\big(J^{\hat{y}_k}_\lambda (\pi^*) - J^{\hat{y}_k}_\lambda (\pi)\big) \label{eq:ppo performance difference}
\end{align}
And with H\"{o}lder's inequality, we have,
\begin{align}
    &\mathbb{E}_{s \sim \nu_{\pi^*}}\big[\mathbb{E}_{a \sim \pi_{\theta_{k}}}[\tau^{-1}_{k+1} f_{\theta_{k+1}}- \tau^{-1}_{k} f_{\theta_{k}}]- \mathbb{E}_{a \sim \pi_{\theta_{k+1}}}\\
    [&\tau^{-1}_{k+1} f_{\theta_{k+1}}-\tau^{-1}_{k} f_{\theta_{k}}]\big] \\
    =&\mathbb{E}_{s \sim \nu_{\pi^*}}\big[\big\langle \tau^{-1}_{k+1} f_{\theta_{k+1}}- \tau^{-1}_{k} f_{\theta_{k}}, \pi_{\theta_{k}}-\pi_{\theta_{k+1}} \big\rangle\big]\label{eq:ppo holder}\\
    \leq& \mathbb{E}_{s \sim \nu_{\pi^*}}\big[\Vert\tau^{-1}_{k+1} f_{\theta_{k+1}}- \tau^{-1}_{k} f_{\theta_{k}}\Vert_{\infty} \Vert\pi_{\theta_{k}}-\pi_{\theta_{k+1}}\Vert_1\big] 
\end{align}
Insert Eqs.~\eqref{eq:ppo performance difference} and~\eqref{eq:ppo holder} into Eq.~\eqref{eq:ppo therorem 1st step}, we have,
\begin{align}
    &\mathbb{E}_{s \sim \nu_{\pi^*}}[\text{KL}(\pi^*\Vert\pi_{\theta_{k+1}})] - \mathbb{E}_{s \sim \nu_{\pi^*}}[\text{KL}(\pi^*\Vert\pi_{\theta_{k}})] \nonumber\\
    \leq& \varepsilon_k - (1-\gamma)\beta^{-1}\big(J^{\hat{y}_k}_\lambda (\pi^*) - J^{\hat{y}_k}_\lambda (\pi)\big) - 1/2\mathbb{E}_{s \sim \nu_{\pi^*}} \nonumber\\
    &\big[\Vert\pi_{\theta_{k+1}} - \pi_{\theta_{k}}\Vert^2_1\big] + \mathbb{E}_{s \sim \nu_{\pi^*}}\big[\Vert\tau^{-1}_{k+1} f_{\theta_{k+1}}- \tau^{-1}_{k} f_{\theta_{k}}\Vert_{\infty} \nonumber\\
    &\Vert\pi_{\theta_{k}}-\pi_{\theta_{k+1}}\Vert_1\big] \nonumber\\
    \leq & \varepsilon_k - (1-\gamma)\beta^{-1}\big(J^{y^*}_\lambda (\pi^*) - J^{\hat{y}_k}_\lambda (\pi) - J^{y^*}_\lambda (\pi^*)  \nonumber\\ 
    &+ J^{\hat{y}_k}_\lambda (\pi^*) \big)+ 1/2\mathbb{E}_{s \sim \nu_{\pi^*}}\big[\Vert\tau^{-1}_{k+1} f_{\theta_{k+1}}- \tau^{-1}_{k} f_{\theta_{k}}\Vert^2_{\infty}\big]\nonumber\\
    \leq & \varepsilon_k - (1-\gamma)\beta^{-1}\big(J^{y^*}_\lambda (\pi^*) - J^{\hat{y}_k}_\lambda (\pi)\big)  \nonumber\\
    &+ (1-\gamma)\beta^{-1}\big(J^{y^*}_\lambda (\pi^*) - J^{\hat{y}_k}_\lambda (\pi^*)\big) \nonumber\\
    &+ 1/2\mathbb{E}_{s \sim \nu_{\pi^*}}\big[\Vert\tau^{-1}_{k+1} f_{\theta_{k+1}}- \tau^{-1}_{k} f_{\theta_{k}}\Vert^2_{\infty}\big].
\end{align}
The second inequality holds by using the inequality $2AB - B^2\leq A^2$, with a minor abuse of notations. Here, $A := \Vert\tau^{-1}_{k+1} f_{\theta_{k+1}}- \tau^{-1}_{k} f_{\theta_{k}}\Vert_{\infty}$ and $B := \Vert\pi_{\theta_{k}}-\pi_{\theta_{k+1}}\Vert_1$.
Then, by plugging in Lemma~\ref{lem:y objective} and Eq.~\eqref{eq:stepwise energy difference} we end up with, 
\begin{align}
    &\mathbb{E}_{s \sim \nu_{\pi^*}}[\text{KL}(\pi^*\Vert\pi_{\theta_{k+1}})] - \mathbb{E}_{s \sim \nu_{\pi^*}}[\text{KL}(\pi^*\Vert\pi_{\theta_{k}})] \nonumber\\
    \leq& \varepsilon_k - (1-\gamma)\beta^{-1}\big(J^{y^*}_\lambda (\pi^*) - J^{\hat{y}_k}_\lambda (\pi_k)\big) \label{eq:ppo theorem aux eq1}\\ 
    &+(1-\gamma)\beta^{-1}\big(\frac{2c_3M(1-\gamma)\lambda}{\sqrt{k}}\big) + (\varepsilon'_k+\beta^{-2}_k U) 
\end{align}
Rearrange Eq.~\eqref{eq:ppo theorem aux eq1}, we have
\begin{align}
    &(1-\gamma)\beta^{-1}\big(J^{y^*}_\lambda (\pi^*) - J^{\hat{y}_k}_\lambda (\pi_k)\big)\nonumber\\
    \leq &\mathbb{E}_{s \sim \nu_{\pi^*}}[\text{KL}(\pi^*\Vert\pi_{\theta_{k}})]-\mathbb{E}_{s \sim \nu_{\pi^*}}[\text{KL}(\pi^*\Vert\pi_{\theta_{k+1}})] \nonumber\\
    &+\big(\frac{2c_3M(1-\gamma)^2\lambda}{\beta\sqrt{k}}\big) +\varepsilon_k+\varepsilon'_k+\beta^{-2}_k U \label{eq:ppo theorem aux eq2}
\end{align}
And then telescoping Eq.~\eqref{eq:ppo theorem aux eq2} results in,
\begin{align}
    &(1-\gamma)\sum^{K}_{k=1}\beta^{-1}\min_{k\in[K]}\big(J^{y^*}_{\lambda}(\pi^*)-J^{\hat{y}_k}_{\lambda}(\pi_k)\big)\nonumber\\
    \leq& (1-\gamma)\sum^{K}_{k=1}\beta^{-1}\big(J^{y^*}_{\lambda}(\pi^*)-J^{\hat{y}_k}_{\lambda}(\pi_k)\big)\nonumber\\
    \leq& \mathbb{E}_{s \sim \nu_{\pi^*}}[\text{KL}(\pi^*\Vert\pi_{0})]-\mathbb{E}_{s \sim \nu_{\pi^*}}[\text{KL}(\pi^*\Vert\pi_{K})] \nonumber\\
    &+ \lambda r_{\max} (1-\gamma)^2\sum^{K}_{k=1}\beta^{-1}\big(\frac{2c_3}{\sqrt{k}}\big) + U \sum^{K}_{k=1}\beta^{-2}_k \nonumber\\ 
    &+ \sum^{K}_{k=1}(\varepsilon_k+\varepsilon'_k) \label{eq:main result ppo}
\end{align}
We complete the final step in Eq.\eqref{eq:main result ppo} by plugging in Lemma~\ref{lem:y objective} and Eq.~\eqref{eq:error propagation}.
Per the observation we make in the proof of Theorem~\ref{th:major result},
\begin{enumerate}
    \item $\mathbb{E}_{s \sim \nu_{\pi^*}}[\text{KL}(\pi^*\Vert\pi_{0})] \leq \log\mathcal{A}$ due to the uniform initialization of policy.
    \item $\text{KL}(\pi^*\Vert\pi_{K})$ is a non-negative term.
\end{enumerate}
We now have,
\begin{align}
    &\min_{k\in[K]}J^{y^*}_{\lambda}(\pi^*)-J^{\hat{y}_k}_{\lambda}(\pi_k) \nonumber\\
    \leq& \frac{\log|\mathcal{A}|+  UK\beta^{-2} + \sum^{K}_{k=1}(\varepsilon_k+\varepsilon'_k}{(1-\gamma)K\beta^{-1}})\nonumber\\
    &+\lambda r_{\max} (1-\gamma)\big(\frac{2c_3}{\sqrt{k}}\big)
\end{align}
Replacing $\beta$ with $\beta_0\sqrt{K}$ finishes the proof.
\end{proof}

\subsection{Proof of Theorem~\ref{th:major result}}
\label{sec:proof of npg}
In the following part, we focus the convergence of neural NPG. We first define the following terms under neural NPG update rule.
\begin{lemma}
\label{lem:neural npg gradient and fisher}
\cite{wang2019neural} For energy-based policy $\pi_\theta$, we have policy gradient and Fisher information matrix,
\begin{align}
    \nabla_{\theta} J(\pi_{\theta}) &= \tau \mathbb{E}_{d_{\pi_\theta}(s,a)}[Q_{\pi_\theta}(s,a) (\phi_\theta (s,a)- \mathbb{E}_{\pi_\theta}[\phi_\theta(s,a')])] \\ 
    F(\theta) &= \tau^2  \mathbb{E}_{d_{\pi_\theta}(s,a)}[(\phi_\theta (s,a) - \mathbb{E}_{\pi_\theta}[\phi_\theta(s,a')])\nonumber\\
    &  (\phi_\theta (s,a) - \mathbb{E}_{\pi_\theta}[\phi_\theta(s,a')])^\top]
\end{align}
\end{lemma}
We then derive an upper bound for $J^{y^*}_{\lambda}(\pi^*)-J^{y^*}_{\lambda}(\pi_k)$ for the neural NPG method in the following lemma:
\begin{lemma}
\label{lem:npg one step pi}
(One-step difference of $\pi$) It holds that, with probability of $1-\delta$,
\begin{align}
    &(1-\gamma)\big(J^{\hat{y}_k}_{\lambda}(\pi^*)-J^{\hat{y}_k}_{\lambda}(\pi_k)\big) \nonumber\\
     \leq  &\eta_{{\rm NPG}}^{-1}\mathbb{E}_{s \sim \nu_{\pi^*}}\big[ \text{KL}(\pi^*\Vert\pi_{k})-\text{KL}(\pi^*\Vert\pi_{k+1})\big]+ \nonumber\\
    & \eta_{{\rm NPG}} (9\Upsilon^2+r_{\max}^2)+2c_0\epsilon'_{k} + \eta_{{\rm NPG}}^{-1}\epsilon''_{k}, \label{lem:bound conclusion} 
\end{align}
%
\begin{align}
    \text{where }&\\
    \epsilon'_{k} =& \mathcal{O}(\Upsilon^{3}m^{-1/2}\log(1/\delta)+\Upsilon^{5/2}m^{-1/4}\sqrt{\log(1/\delta)}\nonumber \\
    &+\Upsilon r_{\max}^2m^{-1/4}+\Upsilon^2K_{{\rm TD}}^{-1/2}+\Upsilon), \\ 
    \epsilon''_k =& 8 \eta_{{\rm NPG}} \Upsilon^{1/2} c_0\sigma_\xi^{1/2} T^{-1/4} \nonumber&&\\
     &+ \mathcal{O}((\tau_{k+1}+\eta_{{\rm NPG}}) \Upsilon^{3/2} m^{-1/4}\nonumber&&\\
    &+ \eta_{{\rm NPG}} \Upsilon^{5/4} m^{-1/8}),\nonumber&&\\
\end{align}
$c_0$ is defined in Assumption~\ref{assu: concentrability regularity} and $\sigma_\xi$ is defined in Assumption~\ref{assu: variance regularity}. Meanwhile, $\Upsilon$ is the radius of the parameter space, $m$ is the width of the neural network, and $T$ is the sample batch size.
\end{lemma}
%

\begin{proof}
We start from the following, 
%
\begin{align}
    & \text{KL}(\pi^*\Vert\pi_{k})-\text{KL}(\pi^*\Vert\pi_{k+1})-\text{KL}(\pi_{k+1}\Vert\pi_{k}) \\
 =& \mathbb{E}_{a \sim \pi^*}\big[\log(\frac{\pi_{k+1}}{\pi_{k}})\big] -\mathbb{E}_{a \sim \pi_{k+1}}\big[\log(\frac{\pi_{k+1}}{\pi_{k}})\big] \label{eq:npg one step main 1}\\ &\text{(by KL's definition)}.
\end{align}

We now show the building blocks of the proof.
\textit{First}, we add and subtract a few terms to RHS of Eq.~\eqref{eq:npg one step main 1} then take the expectation of both sides with respect to $s\sim \nu_{\pi^*}$. Rearrange these terms, we get,
\begin{flalign}
    & \mathbb{E}_{s\sim\nu_{\pi^*}}\big[\text{KL}(\pi^*\Vert\pi_{k})-\text{KL}(\pi^*\Vert\pi_{k+1})-\text{KL}(\pi_{k+1}\Vert\pi_{k})\big] \\
    =& \eta_{{\rm NPG}}\mathbb{E}_{s\sim\nu_{\pi^*}}\big[\mathbb{E}_{a \sim \pi^*}[\Tilde{Q}_{\pi_k,\hat{y}_k}]-\mathbb{E}_{a \sim \pi_{k}}[\Tilde{Q}_{\pi_k,\hat{y}_k}]\big] \\
    &+ H_k \label{eq:npg term list}
\end{flalign}
where $H_k$ is denoted by,
\begin{flalign}
    H_k :=& \mathbb{E}_{s \sim \nu_{\pi^*}}\big[\mathbb{E}_{a \sim \pi^*}[\log(\frac{\pi_{k+1}}{\pi_{k}})-\eta_{{\rm NPG}} \Tilde{Q}_{\omega_k}\big]&&\nonumber \\
    & -\mathbb{E}_{a \sim \pi_{k}}\big[\log(\frac{\pi_{k+1}}{\pi_{k}})-\eta_{{\rm NPG}} \Tilde{Q}_{\omega_k}]\big]&&\\ 
    &+\eta_{{\rm NPG}} \mathbb{E}_{s \sim \nu_{\pi^*}}\big[\mathbb{E}_{a \sim \pi^*}[\Tilde{Q}_{\omega_k}-\Tilde{Q}_{\pi_k, \hat{y}_k}]&&\\ 
    &-\mathbb{E}_{a \sim \pi_{k}}[\Tilde{Q}_{\omega_k}-\Tilde{Q}_{\pi_k, \hat{y}_k}]\big] &&\\
    &+\mathbb{E}_{s \sim \nu_{\pi^*}}\big[\mathbb{E}_{a \sim \pi_{k}}[\log(\frac{\pi_{k+1}}{\pi_{k}})]&&\\
    &-\mathbb{E}_{a \sim \pi_{k+1}}[\log(\frac{\pi_{k+1}}{\pi_{k}})]\big]\label{eq:Hk}&&
\end{flalign}
 By Lemma~\ref{lem:performance difference}, we have
\begin{align}
    &\eta_{{\rm NPG}}\mathbb{E}_{s\sim\nu_{\pi^*}}\big[\mathbb{E}_{a \sim \pi^*}[\Tilde{Q}_{\pi_k,\hat{y}_k}]-\mathbb{E}_{a \sim \pi_{k}}[\Tilde{Q}_{\pi_k,\hat{y}_k}]\big] \\
    =& \eta_{{\rm NPG}}(1-\gamma)\big(J^{\hat{y}_k}_\lambda(\pi^*)-J^{\hat{y}_k}_{\lambda}(\pi_k)\big) \label{eq:npg aux1}
\end{align}

Insert Eqs.~\eqref{eq:npg aux1} 
back to Eq.~\eqref{eq:npg term list}, we have,
\begin{align}
    &\eta_{{\rm NPG}} (1-\gamma)\big(J^{\hat{y}_k}_\lambda(\pi^*)-J^{\hat{y}_k}_{\lambda}(\pi_k)\big) \nonumber \\
    =& \mathbb{E}_{s \sim \nu_{\pi^*}}\big[\text{KL}(\pi^*\Vert\pi_{k})-\text{KL}(\pi^*\Vert\pi_{k+1}) -\text{KL}(\pi_{k+1}\Vert\pi_{k})\big]\nonumber \\ &-H_k \nonumber \\
    \leq& \mathbb{E}_{s \sim \nu_{\pi^*}}\big[\text{KL}(\pi^*\Vert\pi_{k})-\text{KL}(\pi^*\Vert\pi_{k+1}) -\text{KL}(\pi_{k+1}\Vert\pi_{k})\big]\nonumber \\ &+|H_k| \label{eq:bound final1}
\end{align}
We reach the final inequality of Eq.~\eqref{eq:bound final1} by algebraic manipulation. \textit{Second}, we follow Lemma 5.5 of \cite{wang2019neural} and obtain an upper bound for Eq.~\eqref{eq:Hk}. Specifically, with probability of $1-\delta$,
\begin{align}
    &\mathbb{E}_{a \sim {\rm init}}\Big[|H_k|-\mathbb{E}_{s \sim \nu_{\pi^*}}[\text{KL}(\pi_{k+1}\Vert\pi_{k})]\Big] \nonumber \\
    \leq &\eta_{{\rm NPG}}^2 (9\Upsilon^2+r_{\max}^2)+2\eta_{{\rm NPG}} c_0 \epsilon'_{k} + \epsilon''_{k} \label{eq:bound final2}
\end{align}
The expectation is taken over randomness. 
With these building blocks of Eqs.~\eqref{eq:bound final1} and~\eqref{eq:bound final2}, we are now ready to reach the concluding inequality.
Plugging Eqs.~\eqref{eq:bound final2} back into Eq.~\eqref{eq:bound final1}, we end up with, with probability of $1-\delta$,
\begin{align}
    &\eta_{{\rm NPG}} (1-\gamma)\big(J^{\hat{y}_k}_\lambda(\pi^*)-J^{\hat{y}_k}_{\lambda}(\pi_k)\big) \nonumber \\
    \leq& \mathbb{E}_{s \sim \nu_{\pi^*}}\big[\text{KL}(\pi^*\Vert\pi_{k})-\text{KL}(\pi^*\Vert\pi_{k+1}) \big]
    \\&
    + \eta_{{\rm NPG}}^2 (9\Upsilon^2+
    r_{\max}^2)+2\eta_{{\rm NPG}} c_0 \epsilon'_{k} + \epsilon''_{k} \label{eq:bound final3}
\end{align}
Dividing both sides of Eq.~\eqref{eq:bound final3} by $\eta_{{\rm NPG}}$ completes the proof. The details are included in the Appendix.
\end{proof}

We have the following Lemma to bound the error terms $H_k$ defined in Eq.~\eqref{eq:Hk} of Lemma~\ref{lem:npg one step pi}.
\begin{lemma}
\label{lem:Hk}
\cite{wang2019neural}. Under Assumptions~\ref{assu: state distribution regularity}, we have
\begin{align}
    &\mathbb{E}_{a \sim {\rm init}}\Big[|H_k|-\mathbb{E}_{s \sim \nu_{\pi^*}}[\text{KL}(\pi_{k+1}\Vert\pi_{k})]\Big] \nonumber \\
    \leq &\eta_{{\rm NPG}}^2 (9\Upsilon^2+r_{\max}^2)+\eta_{{\rm NPG}}(\varphi'_k+\psi'_k) \epsilon'_{k} + \epsilon''_{k} \\
\end{align}
Here the expectation is taken over all the randomness. We have $\epsilon'_{k}:=\Vert Q_{\omega_k}-Q_{\pi_k}\Vert^2_{\nu_{\pi_k}}$ and 
\begin{align}
    \epsilon''_{k} 
    =& \sqrt{2}\Upsilon^{1/2}\eta_{{\rm NPG}}(\varphi_k+\psi_k)\tau_k^{-1}\big\{\mathbb{E}_{(s,a) \sim \sigma_{\pi_{\theta_k}}}[\Vert\xi_k(\delta_k)\Vert_2^2 ] \nonumber \\
    &+ \mathbb{E}_{(s,a) \sim \sigma_{\pi_{\omega_k}}}[\Vert\xi_k(\omega_k)\Vert_2^2 ]\big\}^{1/2} \\
    &+ \mathcal{O}((\tau_{k+1}+\eta_{{\rm NPG}}) \Upsilon^{3/2} m^{-1/4}
    + \eta_{{\rm NPG}} \Upsilon^{5/4} m^{-1/8}).
\end{align}
Recall $\xi_k(\omega_k)$ and $\xi_k(\omega_k)$ are defined in Assumption~\ref{assu: variance regularity}, while $\varphi_k$,$\psi_k$, $\varphi'_k$, and $\psi_k$ are defined in Assumption~\ref{assu: concentrability regularity}.
\end{lemma}
Please refer to \cite{wang2019neural} for complete proof. Finally, we are ready to show the proof for Theorem~\ref{th:major result}.
%
%

\begin{proof}
First, we combine Lemma~\ref{lem:y objective} and~\ref{lem:npg one step pi} to get the following:
\begin{align}
        &(1-\gamma)\big(J^{y^*}_\lambda(\pi^*)-J^{\hat{y}_k}_\lambda(\pi^*) + J^{\hat{y}_k}_{\lambda}(\pi^*)-J^{\hat{y}_k}_{\lambda}(\pi_k)\big)\nonumber\\
     \leq  &\eta_{{\rm NPG}}^{-1}\mathbb{E}_{s \sim \nu_{\pi^*}}\left[ \text{KL}(\pi^*\Vert\pi_{k})-\text{KL}(\pi^*\Vert\pi_{k+1})\right] \nonumber\\
    &+ \eta_{{\rm NPG}} (9\Upsilon^2+r_{\max}^2)+2 c_0\epsilon'_{k} + \eta_{{\rm NPG}}^{-1}\epsilon''_{k} \nonumber\\
    &+ \frac{2c_3M(1-\gamma)^2\lambda}{\sqrt{k}}\label{eq:npg per step diff}
\end{align}
We can then see this:
\begin{enumerate}
    \item $\mathbb{E}_{s \sim \nu_{\pi^*}}[\text{KL}(\pi^*\Vert\pi_{1})] \leq \log|\mathcal{A}|$ due to the uniform initialization of policy.
    \item $\text{KL}(\pi^*\Vert\pi_{K+1})$ is a non-negative term.
\end{enumerate}
And by setting $\eta_{{\rm NPG}}=1/\sqrt{K}$ and telescoping Eq.~\eqref{eq:npg per step diff}, we obtain,
\begin{align}
    &(1-\gamma)\min_{k\in[K]}\big(J^{y^*}_{\lambda}(\pi^*)-J^{\hat{y}_k}_{\lambda}(\pi_k)\big)\nonumber\\
    \leq& (1-\gamma)\frac{1}{K}\sum_{k=1}^K\mathbb{E}(J^{y^*}_{\lambda}(\pi^*)-J^{\hat{y}_k}_{\lambda}(\pi_k))\nonumber\\
     \leq& \frac{1}{\sqrt{K}}(\mathbb{E}_{s \sim \nu_{\pi^*}}[ \text{KL}(\pi^*\Vert\pi_{1})]+9\Upsilon^2+r_{\max}^2)+ \frac{1}{K}\sum_{k=1}^K\nonumber\\
    & (2\sqrt{K}c_0\epsilon'_{k} + \eta_{{\rm NPG}}^{-1}\epsilon''_{k}+\frac{2c_3M(1-\gamma)^2\lambda}{\sqrt{k}}) \label{eq:npg telescope result}
\end{align}
plug $\epsilon'_{k}$ and $\epsilon''_{k}$ defined in Lemma~\ref{lem:npg one step pi} into Eq.\eqref{eq:npg telescope result},and set $\epsilon_k$ as,
\begin{align}
    \epsilon_k =& \sqrt{8} c_0 \Upsilon^{1/2} \sigma_\xi^{1/2} T^{-1/4} \nonumber&&\\
    &+ \mathcal{O}\big((\tau_{k+1} K^{1/2}+1) \Upsilon^{3/2} m^{-1/4}+ \Upsilon^{5/4} m^{-1/8}\big) \nonumber&&\\
    &+ c_0\mathcal{O}(\Upsilon^{3}m^{-1/2}\log(1/\delta)+\Upsilon^{5/2}m^{-1/4}\sqrt{\log(1/\delta)}\nonumber \\
    &+\Upsilon r_{\max}^2m^{-1/4}+\Upsilon^2K_{{\rm TD}}^{-1/2}+\Upsilon) &&
\end{align}
we complete the proof.
\end{proof}

\subsection{Proof of Lemma~\ref{lem:volatility objective}}
\label{appendix:proof of volatility objective}
\begin{proof}
First, we have 
$\mathbb{E}[G] = \frac{1}{1-\gamma}\mathbb{E}[R]$, i.e., the per-step reward $R$ is an unbiased estimator of the cumulative reward $G$. 
Second, it is proved that $\mathbb{V}(G) \leq \frac{\mathbb{V}(R)}{(1-\gamma)^2}$~\cite{bisi2020riskaverse}. Given $\lambda \ge 0$, summing up the above equality and inequality, we have
\begin{align}
\frac{1}{(1-\gamma)}J_{\frac{\lambda}{(1-\gamma)}} (\pi)  &= \frac{1}{(1-\gamma)}\Big(\mathbb{E}[R] - \frac{\lambda}{(1-\gamma)} \mathbb{V}(R)\Big)\\
&\le \mathbb{E}[G] - \lambda \mathbb{V}(G) 
= J^G_\lambda (\pi).
\end{align}
It completes the proof.
\end{proof}

\subsection{Proof of Lemma~\ref{lem: policy evaluation error}}
\label{sec:proof of policy evaluation error}

We first provide the supporting lemmas for Lemma~\ref{lem: policy evaluation error}.
We define the local linearization of $f((s,a);\theta)$ defined in Eq.~\eqref{eq:PNN def} at the initial point $\Theta_{{\rm init}}$ as,
\begin{align}
    \hat{f}((s,a);\theta)=\frac{1}{\sqrt{m}}\sum_{v=1}^{m} b_v\mathbbm{1}\{[\Theta_{{\rm init}}]_v^\top (s,a)>0\} [\theta]_v^\top (s,a) \\\label{eq:nnl def}
\end{align}
We then define the following function spaces,
\begin{align}
    &\mathcal{F}_{\Upsilon,m}:= \Bigg\{\frac{1}{\sqrt{m}}\sum_{v=1}^{m} b_v\mathbbm{1}\big\{[\Theta_{{\rm init}}]_v^\top (s,a)>0\big\} [\theta]_v^\top (s,a):\nonumber\\
    &\Vert \theta-\Theta_{{\rm init}}\Vert_2 \leq \Upsilon\Bigg\},
\end{align}
and
\begin{align}
    &\bar{\mathcal{F}}_{\Upsilon,m}:= \Bigg\{\frac{1}{\sqrt{m}}\sum_{v=1}^{m} b_v\mathbbm{1}\big\{[\Theta_{{\rm init}}]_v^\top (s,a)>0\big\} [\theta]_v^\top (s,a):\nonumber\\
    &\Vert[\theta]_v-[\Theta_{{\rm init}}]_v\Vert_\infty \leq \Upsilon/\sqrt{md}\Bigg\}.
\end{align}
$[\Theta_{{\rm init}}]_r\sim\mathcal{N}(0,I_d/d)$ and $b_r\sim\text{Unif}(\{-1,1\})$ are the initial parameters. By the definition, $\bar{\mathcal{F}}_{\Upsilon,m}$ is a subset of $\mathcal{F}_{\Upsilon,m}$. The following lemma characterizes the deviation of $\bar{\mathcal{F}}_{\Upsilon,m}$ from $\mathcal{F}_{\Upsilon,\infty}$.
\begin{lemma}
\label{lem:projection error}
(Projection Error)\cite{rahimi2008weighted}. Let $f\in \mathcal{F}_{\Upsilon,\infty}$, where $\mathcal{F}_{\Upsilon,\infty}$ is defined in Assumption~\ref{assu: action-value function class}. For any $\delta>0$, it holds with probability at least $1-\delta$ that
\begin{align}
\Vert\Pi_{\bar{\mathcal{F}}_{\Upsilon,m}}f-f\Vert_\varsigma \leq \Upsilon m^{-1/2}[1+\sqrt{2\log(1/\delta)}]
\end{align}
where $\varsigma$ is any distribution over $S \times A$.
\end{lemma}
Please refer to \cite{rahimi2008weighted} for a detail proof.
\begin{lemma}
\label{lem:linearization error}
(Linearization Error) Under Assumption.~\ref{assu: state distribution regularity}, for all $\theta \in \mathcal{D}$, where $\mathcal{D} = \{\xi \in \mathbb{R}^{md}:\Vert\xi-\Theta_{\text{init}} \Vert_2 \leq \Upsilon \}$, it holds that,
\begin{align}
    \mathbb{E}_{\nu_\pi}\Big[\Big(f\big((s,a);\theta\big)-\hat{f}\big((s,a);\theta\big)\Big)^2\Big] \leq \frac{4c_1 \Upsilon^3}{\sqrt{m}}
\end{align}
where $c_1 = c\sqrt{\mathbb{E}_{\mathcal{N}(0,I_d/d)}[1/\Vert (s,a)\Vert_2^2]}$, and $c$ is defined in Assumption.~\ref{assu: state distribution regularity}.
\end{lemma}
\begin{proof}
We start from the definitions in Eq.~\eqref{eq:PNN def} and Eq.~\eqref{eq:nnl def},
\begin{align}
    &\mathbb{E}_{\nu_\pi}\Big[\Big(f\big((s,a);\theta\big)-\hat{f}\big((s,a);\theta\big)\Big)^2\Big]\nonumber\\
    =&\mathbb{E}_{\nu_\pi}\Big[\Big(\frac{1}{\sqrt{m}}\Big|\sum^m_{v=1}\Big(\big(\mathbbm{1}\{[\theta]_v^\top (s,a)>0\} - \mathbbm{1}\{[\Theta_{{\rm init}}]_v^\top (s,a)\nonumber\\
    &>0\}\big) b_v [\theta]_v^\top (s,a)\Big)\Big|\Big)^2\Big]\nonumber\\
    \leq& \frac{1}{m}\mathbb{E}_{\nu_\pi}\Big[\Big(\sum^m_{v=1}\Big(\Big|\mathbbm{1}\{[\theta]_v^\top (s,a)>0\} - \mathbbm{1}\{[\Theta_{{\rm init}}]_v^\top (s,a)\nonumber\\
    &>0\}\Big|  \Big|b_v\Big| \Big| [\theta]_v^\top (s,a)\Big|\Big)\Big)^2\Big]  \label{eq:le main 1}
\end{align}
The above inequality holds because the fact that $\vert \sum W \vert \leq  \sum \vert W \vert$, where $W = \big((\mathbbm{1}\{[\theta]_v^\top (s,a)>0\} - \mathbbm{1}\{[\Theta_{{\rm init}}]_v^\top (s,a)>0\}) b_v [\theta]_v^\top (s,a)\big)$. 
$\Theta_{{\rm init}}$ is defined in Eq.~\eqref{eq:Theta}.
Next, since $\mathbbm{1}\{[\Theta_{{\rm init}}]_v^\top (s,a)>0\} \neq \mathbbm{1}\{[\theta]_v^\top (s,a)>0\}$, we have,
\begin{align}
    |[\Theta_{{\rm init}}]_v^\top (s,a)| &\leq |[[\theta]_v^\top (s,a) - \Theta_{{\rm init}}]_v^\top (s,a)| \nonumber \\
    &\leq \Vert[\theta]_v-[\Theta_{{\rm init}}]_v\Vert_2 \label{eq:le aux 1},
\end{align}
where we obtain the last inequality from the Cauchy-Schwartz inequality. We also assume that $\Vert (s,a)\Vert_2 \leq 1$ without loss of generality~\cite{liu2019neural,wang2019neural}.
Eq.~\eqref{eq:le aux 1} further implies that,
\begin{align}
    &|\mathbbm{1}\{[\theta]_v^\top (s,a)>0\} - \mathbbm{1}\{[\Theta_{{\rm init}}]_v^\top (s,a)>0\}| \nonumber \\
    \leq& \mathbbm{1}\{|[\Theta_{{\rm init}}]_v^\top (s,a)|\leq \Vert[\theta]_v-[\Theta_{{\rm init}}]_v\Vert_2\} \label{eq:le aux 2}
\end{align}
Then plug Eq.~\eqref{eq:le aux 2} and the fact that $|b_v|\leq 1$ back to Eq.~\eqref{eq:le main 1}, we have the following,
\begin{align}
    &\mathbb{E}_{\nu_\pi}\Big[\Big(f\big((s,a);\theta\big)-\hat{f}\big((s,a);\theta\big)\Big)^2\Big]\nonumber\\
    \leq& \frac{1}{m}\mathbb{E}_{\nu_\pi}\bigg[\bigg(\sum^m_{v=1}\mathbbm{1}\Big\{\Big|[\Theta_{{\rm init}}]_v^\top (s,a)\Big|\leq \Big\Vert[\theta]_v-[\Theta_{{\rm init}}]_v\Big\Vert_2\Big\}\nonumber\\
    & \Big| [\theta]_v^\top (s,a)\Big|\bigg)^2\bigg] \nonumber \\
    \leq& \frac{1}{m}\mathbb{E}_{\nu_\pi}\bigg[\bigg(\sum^m_{v=1}\mathbbm{1}\Big\{\Big|[\Theta_{{\rm init}}]_v^\top (s,a)\Big|\leq \Big\Vert[\theta]_v-[\Theta_{{\rm init}}]_v\Big\Vert_2\Big\}\nonumber \\
    & \Big(\Big|\big([\theta]_v - [\Theta_{{\rm init}}]_v\big)^\top (s,a)\Big| + \Big| [\Theta_{{\rm init}}]_v^\top (s,a)\Big|\Big)\bigg)^2\bigg] \nonumber \\
    \leq& \frac{1}{m}\mathbb{E}_{\nu_\pi}\bigg[\bigg(\sum^m_{v=1}\mathbbm{1}\Big\{\Big|[\Theta_{{\rm init}}]_v^\top (s,a)\Big|\leq \Big\Vert[\theta]_v-[\Theta_{{\rm init}}]_v\Big\Vert_2\Big\}\nonumber \\
    & \Big(\Big\Vert[\theta]_v - [\Theta_{{\rm init}}]_v\Big\Vert_2 + \Big| [\Theta_{{\rm init}}]_v^\top (s,a)\Big|\Big)\bigg)^2\bigg] \nonumber \\
    \leq& \frac{1}{m}\mathbb{E}_{\nu_\pi}\bigg[\bigg(\sum^m_{v=1}\mathbbm{1}\Big\{\Big|[\Theta_{{\rm init}}]_v^\top (s,a)\Big|\leq \Big\Vert[\theta]_v-[\Theta_{{\rm init}}]_v\Big\Vert_2\Big\}\nonumber \\
    & 2\Big\Vert[\theta]_v - [\Theta_{{\rm init}}]_v\Big\Vert_2\bigg)^2\bigg] \label{eq:le main 2}
\end{align}
We obtain the second inequality by the fact that $|A|\leq |A-B|+|B|$. Then follow the Cauchy-Schwartz inequality and $\Vert (s,a)\Vert_2 \leq 1$ we have the third equality. By inserting Eq.~\eqref{eq:le aux 1} we achieve the fourth inequality. We continue Eq.~\eqref{eq:le main 2} by following the Cauchy-Schwartz inequality and plugging $\big\Vert[\theta] - [\Theta_{{\rm init}}]\big\Vert_2 \leq \Upsilon$, 
\begin{align}
     &\mathbb{E}_{\nu_\pi}\Big[\Big(f\big((s,a);\theta\big)-\hat{f}\big((s,a);\theta\big)\Big)^2\Big]\nonumber\\
    \leq& \frac{4\Upsilon^2}{m}\mathbb{E}_{\nu_\pi}\Big[\sum^m_{v=1}\mathbbm{1}\{|[\Theta_{{\rm init}}]_v^\top (s,a)|\leq \Vert[\theta]_v-[\Theta_{{\rm init}}]_v\Vert_2\}\Big]\nonumber \\
    =& \frac{4\Upsilon^2}{m}\sum^m_{v=1}P_{\nu_\pi}|[\Theta_{{\rm init}}]_v^\top (s,a)|\leq \Vert[\theta]_v-[\Theta_{{\rm init}}]_v\Vert_2)\nonumber \\
    \leq& \frac{4c\Upsilon^2}{m}\sum^m_{v=1}\frac{\Vert[\theta]_v-[\Theta_{{\rm init}}]_v\Vert_2}{\Vert \Theta_{{\rm init}}]_v\Vert_2}\nonumber \\
    \leq& \frac{4c\Upsilon^2}{m}\Big(\sum^m_{v=1}\Vert[\theta]_v-[\Theta_{{\rm init}}]_v\Vert_2^2\Big)^{-1/2}\Big(\sum^m_{v=1}\frac{1}{\Vert \Theta_{{\rm init}}]_v\Vert_2^2}\Big)^{-1/2}\nonumber \\
    \leq& \frac{4c_1 \Upsilon^3}{\sqrt{m}} \label{eq:le main 3}
\end{align}
We obtain the second inequality by imposing Assumption~\ref{assu: state distribution regularity} and the third by following the Cauchy-Schwartz inequality. Finally, we set $c_1 := c\sqrt{\mathbb{E}_{\mathcal{N}(0,I_d/d)}[1/\Vert (s,a)\Vert_2^2]} $. Thus, we complete the proof.
\end{proof}
In the $t$-th iterations of TD iteration, we denote the temporal difference terms w.r.t $\hat{f}((s,a);\theta_t)$ and $f((s,a);\theta_t)$ as
\begin{align}
\delta_t^0((s,a),(s,a)';\theta_t) &= \hat{f}((s,a)';\theta_t)-\gamma \hat{f}((s,a);\theta_t) \\
&- r_{s,a},\\
\delta_t^\theta((s,a),(s,a)';\theta_t) &= f((s,a)';\theta_t)-\gamma f((s,a);\theta_t) \\
&- r_{s,a}.
\end{align}
For notation simplicity. in the sequel we write $\delta_t^0((s,a),(s,a)';\theta_t)$ and $\delta_t^\theta((s,a),(s,a)';\theta_t)$ as $\delta_t^0$ and $\delta_t^\theta$. We further define the stochastic semi-gradient $g_t(\theta_t):=\delta_t^\theta \nabla_{\theta} f((s,a);\theta_t)$, its population mean $\bar{g}_t(\theta_t):=\mathbb{E}_{\nu_\pi}[g_t(\theta_t)]$. The local linearization of $\bar{g}_t(\theta_t)$ is $\hat{g}_t(\theta_t):=\mathbb{E}_{\nu_\pi}[\delta_t^0 \nabla_{\theta} \hat{f}((s,a);\theta_t)]$. We denote them as $g_t, \bar{g}_t, \hat{g}_t$ respectively for simplicity.
\begin{lemma}
\label{lem:linearization gradient error}
Under Assumption.~\ref{assu: state distribution regularity}, for all $\theta_t \in \mathcal{D}$, where $\mathcal{D} = \{\xi \in \mathbb{R}^{md}:\Vert\xi-\Theta_{\text{init}} \Vert_2 \leq \Upsilon \}$, it holds with probability of $1-\delta$ that,
\begin{align}
    &\Vert \bar{g}_t-\hat{g}_t\Vert_2 \nonumber\\
    =& \mathcal{O}\Big(\Upsilon^{3/2}m^{-1/4}\big(1+(m\log\frac{1}{\delta})^{-1/2}\big)+\Upsilon^{1/2}r_{\max} m^{-1/4}\Big)
\end{align}
\end{lemma}
\begin{proof}
By the definition of $\bar{g}_t$ and $\hat{g}_t$, we have
\begin{align} 
&\big\Vert \bar{g}_t-\hat{g}_t\big\Vert_2^2\nonumber\\
=&\big\Vert \mathbb{E}_{\nu_\pi}[\delta_t^\theta \nabla_{\theta} f((s,a);\theta_t)-\delta_t^0 \nabla_{\theta} \hat{f}((s,a);\theta_t)]\big\Vert_2^2\nonumber\\
=&\big\Vert \mathbb{E}_{\nu_\pi}[(\delta_t^\theta-\delta_t^0) \nabla_{\theta} f((s,a);\theta_t)+\delta_t^0 (\nabla_{\theta} f((s,a);\theta_t)-\nonumber\\
&\nabla_{\theta} \hat{f}((s,a);\theta_t))]\big\Vert_2^2\nonumber\\
\leq& 2\mathbb{E}_{\nu_\pi}\big[(\delta_t^\theta-\delta_t^0)^2 \Vert\nabla_{\theta} f((s,a);\theta_t)\Vert_2^2\big] + \nonumber\\
&2\mathbb{E}_{\nu_\pi}\big[\big(|\delta_t^0| \Vert \nabla_{\theta} f((s,a);\theta_t)-\nabla_{\theta} \hat{f}((s,a);\theta_t))\Vert_2\big)^2\big] \\\label{eq:ge main 1}
\end{align}
\textit{We obtain the inequality because $(A+B)^2 \leq 2A^2+2B^2$.} We first upper bound $\mathbb{E}_{\nu_\pi}\big[(\delta_t^\theta-\delta_t^0)^2 \Vert\nabla_{\theta} f((s,a);\theta_t)\Vert_2^2\big]$ in Eq.~\eqref{eq:ge main 1}. Since $\Vert (s,a)\Vert_2 \leq 1$, we have $\Vert\nabla_{\theta} f((s,a);\theta_t)\Vert_2 \leq 1$. Then by definition, we have the following first inequality,
\begin{align}
    &\mathbb{E}_{\nu_\pi}\Big[\Big(\delta_t^\theta-\delta_t^0\Big)^2 \Big\Vert\nabla_{\theta} f((s,a);\theta_t)\Big\Vert_2^2\Big] \nonumber\\
    \leq& \mathbb{E}_{\nu_\pi}\Big[\Big(f\big((s,a);\theta_t\big)-\hat{f}\big((s,a);\theta_t\big)-\gamma\Big(f\big((s',a');\theta_t\big)\nonumber\\
    &-\hat{f}\big((s',a');\theta_t)\big)\Big)\Big)^2\Big] \nonumber\\
    \leq& \mathbb{E}_{\nu_\pi}\Big[\Big(\Big|f\big((s,a);\theta_t\big)-\hat{f}\big((s,a);\theta_t\big)\Big|+\Big|f\big((s',a');\theta_t\big)\\
    &-\hat{f}\big((s',a');\theta_t\big)\Big|\Big)^2\Big] \nonumber\\
    \leq& 2\mathbb{E}_{\nu_\pi}\Big[\Big(f\big((s,a);\theta_t\big)-\hat{f}\big((s,a);\theta_t\big)\Big)^2\Big]+2\mathbb{E}_{\nu_\pi}\nonumber\\
    &\Big[\Big(f\big((s',a');\theta_t\big)-\hat{f}\big((s',a');\theta_t\big)\Big)^2\Big] \nonumber\\
    \leq& 4\mathbb{E}_{\nu_\pi}\Big[\Big(f\big((s,a);\theta_t\big)-\hat{f}\big((s,a);\theta_t\big)\Big)^2\Big]\leq\frac{16c_1 \Upsilon^3}{\sqrt{m}} \label{eq:ge aux 1}
\end{align}
We obtain the second inequality by $|\gamma| \leq 1$, \textit{then obtain the third inequality by the fact that $(A+B)^2 \leq 2A^2+2B^2$.} We reach the final step by inserting Lemma~\ref{lem:linearization error}. We then proceed to upper bound $\mathbb{E}_{\nu_\pi}\big[|\delta_t^0| \Vert \nabla_{\theta} f((s,a);\theta_t)-\nabla_{\theta} \hat{f}((s,a);\theta_t))\Vert_2\big]$. From H\"{o}lder's inequality, we have, 
\begin{align}
    &\mathbb{E}_{\nu_\pi}\big[\big(|\delta_t^0| \Vert \nabla_{\theta} f((s,a);\theta_t)-\nabla_{\theta} \hat{f}((s,a);\theta_t))\Vert_2\big)^2\big] \nonumber \\
    \leq &\mathbb{E}_{\nu_\pi}\big[(\delta_t^0)^2\big] \mathbb{E}_{\nu_\pi}\big[\Vert \nabla_{\theta} f((s,a);\theta_t)-\nabla_{\theta} \hat{f}((s,a);\theta_t))\Vert_2^2\big]\nonumber\\ \label{eq:ge aux 2}
\end{align}
We first derive an upper bound for first term in Eq.\eqref{eq:ge aux 2}, starting from its definition,
\begin{align}
    &\mathbb{E}_{\nu_\pi}\big[(\delta_t^0)^2\big]\nonumber\\
    =&\mathbb{E}_{\nu_\pi}\Big[\big[\hat{f}\big((s',a');\theta_t\big)-\gamma \hat{f}\big((s,a);\theta_t\big) - r_{s,a}\big]^2\Big]\nonumber\\
    \leq&3\mathbb{E}_{\nu_\pi}\Big[\big(\hat{f}\big((s',a');\theta_t\big)\big)^2\Big]+3\mathbb{E}_{\nu_\pi}\Big[\big(\gamma \hat{f}\big((s,a);\theta_t\big)\big)^2\Big] \\
    &+3\mathbb{E}_{\nu_\pi}\Big[ r^2_{s,a}\Big]\nonumber\\
    \leq& 6\mathbb{E}_{\nu_\pi}\Big[\big(\hat{f}\big((s,a);\theta_t\big)\big)^2\Big]+3r_{\max}^2 \nonumber\\
    = & 6\mathbb{E}_{\nu_\pi}\Big[\big(\hat{f}\big((s,a);\theta_t\big)-\hat{f}\big((s,a);\theta_{\pi^*}\big)+\hat{f}\big((s,a);\theta_{\pi^*}\big)\nonumber\\
     &-Q_\pi+Q_\pi\big)^2\Big]+3r_{\max}^2 \nonumber\\
     \leq & 18\mathbb{E}_{\nu_\pi}\Big[\big(\hat{f}\big((s,a);\theta_t\big)-\hat{f}\big((s,a);\theta_{\pi^*}\big)\big)^2\Big] + 18\mathbb{E}_{\nu_\pi}\nonumber\\
     &\Big[\big(\hat{f}\big((s,a);\theta_{\pi^*}\big)-Q_\pi\big)^2\Big] + 18\mathbb{E}_{\nu_\pi}\Big[\big(Q_\pi\big)^2\Big]+ 3r_{\max}^2 \nonumber\\
     \leq& 72\Upsilon^2 + 18\mathbb{E}_{\nu_\pi}\Big[\big(\hat{f}\big((s,a);\theta_{\pi^*}\big)-Q_\pi\big)^2\Big]\nonumber\\
     &+ 21(1-\gamma)^{-2}r_{\max}^2 \label{eq:ge aux 3}
\end{align}
\textit{We obtain the first and the third inequality by the fact that $(A+B+C)^2 \leq 3A^2+3B^2+3C^2$.}
Recall $r_{\max}$ is the boundary for reward function $r$, which leads to the second inequality. We obtain the last inequality in Eq.~\eqref{eq:ge aux 3} following the fact that $|\hat{f}((s,a);\theta_t)-\hat{f}((s,a);\theta_{\pi^*})| \leq \Vert \theta_t-\theta_{\pi^*}\Vert \leq 2\Upsilon$ and $Q_\pi \leq (1-\gamma)^{-1}r_{\max}$. Since $\bar{\mathcal{F}}_{\Upsilon,m} \subset \mathcal{F}_{\Upsilon,m}$, by Lemma~\ref{lem:projection error}, we have,
\begin{align}
    {E}_{\nu_\pi}\Big[\Big(\hat{f}\big((s,a);\theta_{\pi^*}\big)-Q_\pi\Big)^2\Big] \leq \frac{\Upsilon^2\big(1+\sqrt{2\log(1/\delta)}\big)^2}{m} \\\label{eq:ge aux 4}
\end{align}
Combine Eq.~\eqref{eq:ge aux 3} and Eq.~\eqref{eq:ge aux 4}, 
we have with probability of $1-\delta$, 
\begin{align}
    &\mathbb{E}_{\nu_\pi}\big[(\delta_t^0)^2\big]\nonumber\\
    \leq& 72\Upsilon^2(1+\frac{\log(1/\delta)}{m})+ 21(1-\gamma)^{-2}r_{\max}^2 \label{eq:ge main 2}
\end{align}
Lastly we have 
\begin{align}
    &\mathbb{E}_{\nu_\pi}\big[\Vert \nabla_{\theta} f((s,a);\theta_t)-\nabla_{\theta} \hat{f}((s,a);\theta_t))\Vert_2^2\big]\nonumber\\
    =& \mathbb{E}_{\nu_\pi}\Big[\Big(\frac{1}{m}\sum^m_{v=1}\big(\mathbbm{1}\{[\theta]_v^\top (s,a)>0\} - \mathbbm{1}\{[\Theta_{{\rm init}}]_v^\top (s,a)\nonumber\\
    &>0\})^2 (b_v)^2 \Vert (s,a)\Vert_2^2\Big)\Big]\nonumber\\
    \leq& \mathbb{E}_{\nu_\pi}\Big[\frac{1}{m}\sum^m_{v=1}\big(\mathbbm{1}\{|[\Theta_{{\rm init}}]_v^\top (s,a)|\leq \Vert[\theta]_v-[\Theta_{{\rm init}}]_v\Vert_2\}\big)\Big]\nonumber\\
    \leq& \frac{c_1 \Upsilon}{\sqrt{m}} \label{eq:ge aux 5}
    \end{align}
We obtain the first inequality by following Eq.~\eqref{eq:le aux 2} and the fact that $|b_v| \leq 1$ and $\Vert (s,a)\Vert_2  \leq 1$. Then for the rest, we follow the similar argument in Eq.~\eqref{eq:le main 3}. To finish the proof, we plug Eq.~\eqref{eq:ge aux 1}, Eq.~\eqref{eq:ge main 2} and Eq.~\eqref{eq:ge aux 5} back to Eq.~\eqref{eq:ge main 1},
\begin{align}
    &\Vert \bar{g}_t-\hat{g}_t\Vert_2^2\nonumber\\
    \leq & 2\Big(\frac{16c_1 \Upsilon^3}{\sqrt{m}} + \Big(72\Upsilon^2(1+\frac{\log(1/\delta)}{m})+ 21(1-\gamma)^{-2}r_{\max}^2\Big)\nonumber\\
    &\frac{c_1 \Upsilon}{\sqrt{m}}\Big) \nonumber\\
    =&\frac{176 c_1 \Upsilon^3}{\sqrt{m}} + \frac{144 c_1 \Upsilon^3\log(1/\delta)}{m^{3/2}}+ \frac{42 c_1 \Upsilon r_{\max}^2}{(1-\gamma)^{-2}\sqrt{m}}
\end{align}
Then we have,
\begin{align}
    &\Vert \bar{g}_t-\hat{g}_t\Vert_2\nonumber\\
    \leq& \sqrt{\frac{176 c_1 \Upsilon^3}{\sqrt{m}} + \frac{144 c_1 \Upsilon^3\log(1/\delta)}{m^{3/2}}+ \frac{42 c_1 \Upsilon r_{\max}^2}{(1-\gamma)^{-2}\sqrt{m}}} \nonumber\\
    \leq& \sqrt{\frac{176 c_1 \Upsilon^3}{\sqrt{m}}} + \sqrt{\frac{144 c_1 \Upsilon^3\log(1/\delta)}{m^{3/2}}}+ \sqrt{\frac{42 c_1 \Upsilon r_{\max}^2}{(1-\gamma)^{-2}\sqrt{m}}} \nonumber\\
    =&\mathcal{O}\Big(\Upsilon^{3/2}m^{-1/4}\big(1+(m\log\frac{1}{\delta})^{-1/2}\big)+\Upsilon^{1/2}r_{\max} m^{-1/4}\Big)
\end{align}
\end{proof}
Next, we provide the following lemma to characterize the variance of $g_t$.
\begin{lemma}
\label{lem:variance of SUV}
(Variance of the Stochastic Update Vector)\cite{liu2019neural}.There exists a constant $\xi_g^2=\mathcal{O}(\Upsilon^2)$ independent of $t$. such that for any $t \leq T$, it holds that
\begin{align}
    \mathbb{E}_{\nu_\pi}[\Vert g_t(\theta_t)-\bar{g}_t(\theta_t)\Vert_2^2] \leq \xi_g^2
\end{align}
\end{lemma}
A detailed proof can be found in \cite{liu2019neural}. Now we provide the proof for Lemma~\ref{lem: policy evaluation error}.
\begin{proof}
\begin{align}
    &\big\Vert \theta_{t+1}-\theta_{\pi^*}\big\Vert_2^2\nonumber\\
     = &\big\Vert\Pi_\mathcal{D}(\theta_t-\eta g_t(\theta_t))-\Pi_\mathcal{D}(\theta_{\pi^*}-\eta \hat{g}_t(\theta_{\pi^*}))\big\Vert_2^2\nonumber\\
     \leq &\big\Vert (\theta_t-\theta_{\pi^*}) -\eta \big(g_t(\theta_t)-\hat{g}_t(\theta_{\pi^*})\big)\big\Vert_2^2\nonumber\\
     =& \big\Vert \theta_t-\theta_{\pi^*}\big\Vert_2^2 - 2\eta\big( g_t(\theta_t)-\hat{g}_t(\theta_{\pi^*})\big)^\top\big(\theta_t-\theta_{\pi^*}\big)\nonumber\\
     & + \eta^2 \big\Vert g_t(\theta_t)-\hat{g}_t(\theta_{\pi^*})\big\Vert_2^2 \label{eq:pe main 1}
\end{align}
\textit{The inequality holds due to the definition of $\Pi_\mathcal{D}$.} We first upper bound $\big\Vert g_t(\theta_t)-\hat{g}_t(\theta_{\pi^*})\big\Vert_2^2$ in Eq.~\eqref{eq:pe main 1},
\begin{align}
    &\big\Vert g_t(\theta_t)-\hat{g}_t(\theta_{\pi^*})\big\Vert_2^2\nonumber\\
    =&\big\Vert g_t(\theta_t)-\bar{g}_t(\theta_t) +\bar{g}_t(\theta_t)-\hat{g}_t(\theta_t)+\hat{g}_t(\theta_t)- \hat{g}_t(\theta_{\pi^*})\big\Vert_2^2\nonumber\\
    \leq&3\Big(\big\Vert g_t(\theta_t)-\bar{g}_t(\theta_t)\big\Vert_2^2 +\big\Vert\bar{g}_t(\theta_t)-\hat{g}_t(\theta_t)\big\Vert_2^2+\nonumber\\ &\big\Vert\hat{g}_t(\theta_t)-\hat{g}_t(\theta_{\pi^*})\big\Vert_2^2\Big) \label{eq:pe aux 1}
\end{align}
\textit{The inequality holds due to fact that $(A+B+C)^2 \leq 3A^2+3B^2+3C^2$.} Two of the terms on the right hand side of Eq.~\eqref{eq:pe aux 1} are characterized in Lemma~\ref{lem:linearization gradient error} and Lemma~\ref{lem:variance of SUV}. We therefore characterize the remaining term,
\begin{align}
    &\big\Vert\hat{g}_t(\theta_t)-\hat{g}_t(\theta_{\pi^*})\big\Vert_2^2\nonumber\\
    =&\mathbb{E}_{\nu_\pi}\Big[\big(\delta_t^0(\theta_t)-\delta_t^0(\theta_{\pi^*})\big)^2\big\Vert\nabla_{\theta}\hat{f}\big((s,a);\theta_t\big)\big\Vert_2^2\Big]\nonumber\\
    \leq&\mathbb{E}_{\nu_\pi}\bigg[\bigg(\Big(\hat{f}\big((s,a);\theta_t\big)-\hat{f}\big((s,a);\theta_{\pi^*}\big)\Big)-\gamma\Big(\hat{f}\big((s',a');\nonumber\\
    &\theta_t\big)-\hat{f}\big((s',a');\theta_{\pi^*}\big)\Big)\bigg)^2\bigg]\nonumber\\
    \leq&\mathbb{E}_{\nu_\pi}\Big[\Big(\hat{f}\big((s,a);\theta_t\big)-\hat{f}\big((s,a);\theta_{\pi^*}\big)\Big)^2\Big]+2\gamma\mathbb{E}_{\nu_\pi}\nonumber\\
    &\Big[\Big(\hat{f}\big((s',a');\theta_t\big)-\hat{f}\big((s',a');\theta_{\pi^*}\big)\Big)\Big(\hat{f}\big((s,a);\theta_t\big)\nonumber\\
    &-\hat{f}\big((s,a);\theta_{\pi^*}\big)\Big)\Big]\nonumber\\
    &+\gamma^2\mathbb{E}_{\nu_\pi}\Big[\Big(\hat{f}\big((s',a');\theta_t\big)-\hat{f}\big((s',a');\theta_{\pi^*}\big)\Big)^2\Big] \label{eq:pe aux 2}
\end{align}
We obtain the first inequality by the fact that $\Vert\nabla_{\theta}\hat{f}((s,a);\theta_t)\Vert_2 \leq 1$. Then we use the fact that $(s, a)$ and $(s',a')$ have the same marginal distribution as well as $\gamma < 1$ for the second inequality. Follow the Cauchy-Schwarz inequality and the fact that $(s,a)$ and $(s',a')$ have the same marginal distribution, we have
\begin{align}
    &\mathbb{E}_{\nu_\pi}\Big[\Big(\hat{f}\big((s',a');\theta_t\big)-\hat{f}\big((s',a');\theta_{\pi^*}\big)\Big)\Big(\hat{f}\big((s,a);\theta_t\big)-\nonumber \\
    &\hat{f}\big((s,a);\theta_{\pi^*}\big)\Big)\Big]\nonumber \\
    \leq  &\mathbb{E}_{\nu_\pi}\Big[\Big(\hat{f}\big((s',a');\theta_t\big)-\hat{f}\big((s',a');\theta_{\pi^*}\big)\Big)\Big]\mathbb{E}_{\nu_\pi}\nonumber \\
    &\Big[\Big(\hat{f}\big((s,a);\theta_t\big)-\hat{f}\big((s,a);\theta_{\pi^*}\big)\Big)\Big]\nonumber \\
    =&\mathbb{E}_{\nu_\pi}\Big[\Big(\hat{f}\big((s',a');\theta_t\big)-\hat{f}\big((s',a');\theta_{\pi^*}\big)\Big)^2\Big] \label{eq:pe aux 4}
\end{align}
We plug Eq.~\eqref{eq:pe aux 4} back to Eq.~\eqref{eq:pe aux 2},
\begin{align}
    &\big\Vert\hat{g}_t(\theta_t)-\hat{g}_t(\theta_{\pi^*})\big\Vert_2^2 \nonumber \\
    \leq& (1+\gamma)^2\mathbb{E}_{\nu_\pi}\Big[\Big(\hat{f}\big((s,a);\theta_t\big)-\hat{f}\big((s,a);\theta_{\pi^*}\big)\Big)^2\Big]. \label{eq:pe aux 5}
\end{align}
Next, we upper bound $\big( g_t(\theta_t)-\hat{g}_t(\theta_{\pi^*})\big)^\top\big(\theta_t-\theta_{\pi^*}\big)$. We have,
\begin{align}
    &\big( g_t(\theta_t)-\hat{g}_t(\theta_{\pi^*})\big)^\top\big(\theta_t-\theta_{\pi^*}\big) \nonumber\\
    =&\big( g_t(\theta_t)-\bar{g}_t(\theta_t))\big)^\top\big(\theta_t-\theta_{\pi^*}\big) + \big( \bar{g}_t(\theta_t)-\hat{g}_t(\theta_t)\big)^\top\nonumber\\
    &\big(\theta_t-\theta_{\pi^*}\big) + \big( \hat{g}_t(\theta_t)-\hat{g}_t(\theta_{\pi^*})\big)^\top\big(\theta_t-\theta_{\pi^*}\big) \label{eq:pe aux 3}
\end{align}
One term on the right hand side of Eq.~\eqref{eq:pe aux 3} are characterized by Lemma~\ref{lem:variance of SUV}. We continue to characterize the remaining terms. First, by H\"{o}lder's inequality, we have
\begin{align}
    &\big( \bar{g}_t(\theta_t)-\hat{g}_t(\theta_t)\big)^\top\big( \theta_t-\theta_{\pi^*}\big) \nonumber\\
    \geq& -\big\Vert \bar{g}_t(\theta_t)-\hat{g}_t(\theta_t)\big\Vert_2\big\Vert  \theta_t-\theta_{\pi^*}\big\Vert_2 \nonumber\\
    \geq& -2\Upsilon \Vert \bar{g}_t(\theta_t)-\hat{g}_t(\theta_t)\big\Vert_2 \label{eq:pe aux 6}
\end{align}
We obtain the second inequality since $\big\Vert  \theta_t-\theta_{\pi^*}\big\Vert_2 \leq 2\Upsilon$ by definition. For the last term,
\begin{align}
    &\big( \hat{g}_t(\theta_t)-\hat{g}_t(\theta_{\pi^*})\big)^\top\big(\theta_t-\theta_{\pi^*}\big) \nonumber\\
    =&\mathbb{E}_{\nu_\pi}\bigg[\bigg(\Big(\hat{f}\big((s,a);\theta_t\big)-\hat{f}\big((s,a);\theta_{\pi^*}\big)\Big)-\gamma\Big(\hat{f}\big((s',a');\theta_t\big)\nonumber\\
    &-\hat{f}\big((s',a');\theta_{\pi^*}\big)\Big)\bigg)\Big(\nabla_{\theta}\hat{f}\big((s,a);\theta_t\big)\Big)^\top\Big(\theta_t-\theta_{\pi^*}\Big)\bigg] \nonumber\\
    =&\mathbb{E}_{\nu_\pi}\bigg[\bigg(\Big(\hat{f}\big((s,a);\theta_t\big)-\hat{f}\big((s,a);\theta_{\pi^*}\big)\Big)-\gamma\Big(\hat{f}\big((s',a');\theta_t\big)\nonumber\\
    &-\hat{f}\big((s',a');\theta_{\pi^*}\big)\Big)\bigg)\Big(\hat{f}\big((s,a);\theta_t\big)-\hat{f}\big((s,a);\theta_{\pi^*}\big)\Big)\bigg]\nonumber\\
    \geq& \mathbb{E}_{\nu_\pi}\bigg[\bigg(\Big(\hat{f}\big((s,a);\theta_t\big)-\hat{f}\big((s,a);\theta_{\pi^*}\big)\Big)\bigg)^2\bigg]-\nonumber\\
    &\gamma\mathbb{E}_{\nu_\pi}\bigg[\bigg(\Big(\hat{f}\big((s,a);\theta_t\big)-\hat{f}\big((s,a);\theta_{\pi^*}\big)\Big)\bigg)^2\bigg]\nonumber\\
    =& (1-\gamma)\mathbb{E}_{\nu_\pi}\Big[\Big(\hat{f}\big((s,a);\theta_t\big)-\hat{f}\big((s,a);\theta_{\pi^*}\big)\Big)^2\Big], \label{eq:pe aux 7}
\end{align}
where the inequality follows from Eq.~\eqref{eq:pe aux 4}. Combine Eqs.~\eqref{eq:pe main 1}, \eqref{eq:pe aux 1}, \eqref{eq:pe aux 5}, \eqref{eq:pe aux 3}, \eqref{eq:pe aux 6} and ~\eqref{eq:pe aux 7}, we have, 
\begin{align}
    &\big\Vert \theta_{t+1}-\theta_{\pi^*}\big\Vert_2^2\nonumber\\
    \leq& \big\Vert \theta_t-\theta_{\pi^*}\big\Vert_2^2 -\big(2\eta(1-\gamma)-3\eta^2(1+\gamma)^2\big)\nonumber\\
    &\mathbb{E}_{\nu_\pi}\Big[\Big(\hat{f}\big((s,a);\theta_t\big)-\hat{f}\big((s,a);\theta_{\pi^*}\big)\Big)^2\Big]\nonumber\\
    & + 3\eta^2\Vert \bar{g}_t-\hat{g}_t\Vert_2^2 +4\eta \Upsilon \Vert \bar{g}_t-\hat{g}_t\Vert_2 + 4\Upsilon\eta|\xi_g| \nonumber\\
    &+ 3\eta^2\xi_g^2 \label{eq:pe main 2}
\end{align}
We then bound the error terms by rearrange Eq.~\eqref{eq:pe main 2}. First, we have, with probability of $1-\delta$,
\begin{align}
    &\mathbb{E}_{\nu_\pi}\Big[\Big(f\big((s,a);\theta_t\big)-\hat{f}\big((s,a);\theta_{\pi^*}\big)\Big)^2\Big] \nonumber \\
    =&\mathbb{E}_{\nu_\pi}\Big[\Big(f\big((s,a);\theta_t\big)-\hat{f}\big((s,a);\theta_t\big)+\hat{f}\big((s,a);\theta_t\big)\nonumber \\
    &-\hat{f}\big((s,a);\theta_{\pi^*}\big)\Big)^2\Big]\nonumber \\
    \leq &2\mathbb{E}_{\nu_\pi}\Big[\Big(f\big((s,a);\theta_t\big)-\hat{f}\big((s,a);\theta_t\big)\Big)^2+\Big(\hat{f}\big((s,a);\theta_t\big)\nonumber \\
    &-\hat{f}\big((s,a);\theta_{\pi^*}\big)\Big)^2\Big]\nonumber \\
    \leq &\big(\eta(1-\gamma)-1.5\eta^2(1+\gamma)^2\big)^{-1}\Big(\big\Vert \theta_t-\theta_{\pi^*}\big\Vert_2^2 \nonumber \\
    &- \Vert \theta_{t+1}-\theta_{\pi^*}\big\Vert_2^2 + 4\Upsilon\eta|\xi_g| + 3\eta^2\xi_g^2\Big) + \epsilon_g \label{eq:pe main 3}
\end{align}
where 
\begin{align}
    \epsilon_g &= \mathcal{O}(\Upsilon^{3}m^{-1/2}\log(1/\delta)+\Upsilon^{5/2}m^{-1/4}\sqrt{\log(1/\delta)}\nonumber \\
    &+\Upsilon r_{\max}^2m^{-1/4})
\end{align}
We obtain the first inequality by the fact that $(A+B)^2\leq 2A^2 + 2B^2$. Then by Eq.~\eqref{eq:pe main 2}, Lemma~\ref{lem:linearization error} and Lemma~\ref{lem:linearization gradient error}, we reach the final inequality. By telescoping Eq.~\eqref{eq:pe main 3} for $t = $ to $T$, we have, with probability of $1-\delta$,
\begin{align}
    &\big\Vert f\big((s,a);\theta_T\big) - \hat{f}\big((s,a);\theta_{\pi^*}\big) \big\Vert^2 \nonumber \\
    \leq& \frac{1}{T}\sum_{t=1}^{T}\mathbb{E}_{\nu_\pi}\Big[\Big(f\big((s,a);\theta_t\big) - \hat{f}\big((s,a);\theta_{\pi^*}\big)\Big)^2\Big]\nonumber \\
    \leq&T^{-1}\big(2\eta(1-\gamma)-3\eta^2(1+\gamma)^2\big)^{-1}(\Vert \Theta_{{\rm init}}-\theta_{\pi^*} \Vert +\nonumber \\
    &  4\Upsilon T\eta|\xi_g|+3T\eta^2\xi_g^2) + \epsilon_g \label{eq:pe main 4}
\end{align}
Set $\eta=\min\{1/\sqrt{T}, (1-\gamma)/3(1+\gamma)^2\}$, which implies that $T^{-1/2}(2\eta(1-\gamma)-3\eta^2(1+\gamma)^2)^{-1} \leq 1/(1-\gamma)^2$, then we have, with probability of $1-\delta$,
\begin{align}
    &\big\Vert f\big((s,a);\theta_T\big) - \hat{f}\big((s,a);\theta_{\pi^*}\big) \big\Vert \nonumber \\
    \leq&\frac{1}{(1-\gamma)^2\sqrt{T}}\big(\Vert \Theta_{{\rm init}}-\theta_{\pi^*} \Vert_2^2 + 4\Upsilon\sqrt{T}|\xi_g|\nonumber \\
    & +3\xi_g^2\big) + \epsilon_g \nonumber \\
    \leq&\frac{\Upsilon^2 + 4\Upsilon\sqrt{T}|\xi_g|+3\xi_g^2}{(1-\gamma)^2\sqrt{T}} + \epsilon_g \nonumber \\
    =& \mathcal{O}(\Upsilon^{3}m^{-1/2}\log(1/\delta)+\Upsilon^{5/2}m^{-1/4}\sqrt{\log(1/\delta)}\nonumber \\
    &+\Upsilon r_{\max}^2m^{-1/4}+\Upsilon^2T^{-1/2}+\Upsilon)
\end{align}
We obtain the second inequality by the fact that $\Vert \Theta_{{\rm init}}-\theta_{\pi^*} \Vert_2\leq \Upsilon$. Then by definition we replace $\Tilde{Q}_{\omega_k}$ and $\Tilde{Q}_{\pi_k}$
\end{proof}







\label{sec:appendix:related}

\bibliography{appendixSTOPS}
\bibliographystyle{unsrt}
\end{document}